\documentclass{article}
\usepackage[utf8]{inputenc}

\usepackage{algorithm}
\usepackage{algorithmicx}
\usepackage{algpseudocode}
\usepackage{booktabs}

\usepackage{enumerate}
\usepackage{subcaption}
\usepackage{mdframed}
\usepackage{natbib}
\usepackage{hyperref}
\hypersetup{
    colorlinks = true,
	citecolor=blue,
	linkcolor=blue
}
\usepackage{appendix}
\usepackage{geometry}
\usepackage{adjustbox}

\usepackage{xcolor,colortbl}

\newcommand{\dtw}{\text{DTW}}
\newcommand{\idtw}{\text{IDTW}}

\usepackage[whole]{bxcjkjatype}

\usepackage{authblk}
\usepackage{enumerate}
\usepackage{amsthm}
\theoremstyle{definition}
\newtheorem{theorem}{Theorem}
\newtheorem*{theorem*}{Theorem}

\newtheorem*{note*}{Note}
\newtheorem{ex}{Example}
\newtheorem{coro}{Corollary}
\newtheorem{prop}{Proposition}

\usepackage{fourier}
\usepackage{enumerate}
\usepackage{bbm}
\usepackage{wrapfig}
\usepackage{rotating}
\usepackage{tabularx}
\usepackage{multirow}
\usepackage{booktabs}
\usepackage{threeparttable}
\usepackage{amsmath,amssymb,bbm}
\usepackage{comment}

\newcommand{\first}[1]{\textbf{#1}}
\newcommand{\second}[1]{\underline{#1}}

\newcommand{\argmin}{\mathop{\arg\min}}
\newcommand{\bs}{\boldsymbol}
\newcommand{\diff}{\mathrm{d}}

\definecolor{Gray}{gray}{0.85}

\newcommand{\setF}{\mathcal{F}}
\newcommand{\setDn}{\mathcal{D}_n}
\newcommand{\proj}{\mathcal{P}}
\newcommand{\integralB}{\mathcal{B}}
\newcommand{\event}{\mathsf{E}_n}

\title{Improving Nonparametric Classification via Local Radial Regression with an Application to Stock Prediction}
\author[1,2]{Ruixing Cao$^*$}
\author[3,2]{Akifumi Okuno$^*$}
\author[4]{Kei Nakagawa}
\author[1,2]{Hidetoshi Shimodaira}
\affil[1]{Graduate School of Informatics, Kyoto University}
\affil[2]{RIKEN Center for Advanced Intelligence Project}
\affil[3]{The Institute of Statistical Mathematics}
\affil[4]{Innovation Lab, NOMURA Asset Management Co., Ltd.}
\date{\empty}

\begin{document}
\maketitle

\def\thefootnote{*}\footnotetext{A. Okuno is a corresponding author and his e-mail address is okuno@ism.ac.jp. The first two authors (R. Cao and A. Okuno) contributed equally to this work. }

\begin{abstract}
For supervised classification problems, this paper considers estimating the query's label probability through local regression using observed covariates. 
Well-known nonparametric kernel smoother and $k$-nearest neighbor ($k$-NN) estimator, which take label average over a ball around the query, are consistent but asymptotically biased particularly for a large radius of the ball. 
To eradicate such bias, local polynomial regression~(LPoR) and multiscale $k$-NN~(MS-$k$-NN) learn the bias term by local regression around the query and extrapolate it to the query itself. 
However, their theoretical optimality has been shown for the limit of the infinite number of training samples.
For correcting the asymptotic bias with fewer observations, this paper proposes a \emph{local radial regression~(LRR)} and its logistic regression variant called \emph{local radial logistic regression~(LRLR)}, by combining the advantages of LPoR and MS-$k$-NN. The idea is quite simple: we fit the local regression to observed labels by taking only the radial distance as the explanatory variable and then extrapolate the estimated label probability to zero distance. 
The usefulness of the proposed method is shown theoretically and experimentally. We prove the convergence rate of the $L^2$ risk for LRR with reference to MS-$k$-NN, and our numerical experiments, including real-world datasets of daily stock indices, demonstrate that LRLR outperforms LPoR and MS-$k$-NN. 
\end{abstract}

\textbf{Keywords:} Nonparametric classification, Bias correction, Stock prediction

\section{Introduction} \label{sec:intro}
Predicting a query's label from observed pairs of covariates and their labels is called \emph{supervised classification}, and it has played an indispensable role in machine learning and statistics for long decades. For instance, medical diagnosis~\citep{soni2011predictive} predicts disease from a medical image, spam-mail filtering~\citep{cristianini2000introduction} determines whether a mail is spam or not via texts in the mail, and so on; as is well known, there is a wide range of real-world applications.

For this supervised classification problem, many existing studies employ a simple two-step procedure: One first estimates the label probabilities by regression methods, and second applies a classifier so as to output some labels whose estimated probabilities are large enough~\citep{devroye1996probabilistic,samworth2012optimal}. 
The regression in the first step is compatible with arbitrary regression methods. 
To name a few, we may employ a logistic regression~\citep{hastie2001elements,bishop2006PRML}, kernel smoother~\citep{nadaraya1964estimating,tsybakov2009introduction}, $k$-nearest neighbor ~($k$-NN) estimator~\citep{fix1951nonparametric,cover1967nearest}, and so forth including highly-expressive parametric approaches using neural networks~\citep[NN;][]{dreiseitl2002logistic,Goodfellow-et-al-2016}. 
As they each have their own strengths, we need to choose one or aggregate some of them as circumstances dictate.

Amongst these regression approaches, an advantage of nonparametric methods such as kernel smoother and $k$-NN is that they possess both theoretical tractability and high expressive power, i.e., they can approximate any form of $\eta(X)=\mathbb{E}(Y\mid X) =\mathbb{P}(Y=1 \mid X) \in [0,1]$ representing the label probability for a query $X$ when $\eta$ is a Lipschitz function; $Y\in\{0,1\}$ is a binary response variable indicating the label if $Y=1$.
Here $\mathbb{E}(\cdot|X)$ and $\mathbb{P}(\cdot|X)$ denote the conditional expectation and probability, respectively, given $X$.
The asymptotic higher-order biases 
of kernel smoother and $k$-NN
are keenly evaluated and the bias correction has also been developed. 
In theory, the bias-corrected estimators attain faster convergence than highly-expressive parametric models such as neural networks, when the specific form of the ground-truth function $\eta$ is unknown. 
See, e.g., \citet{fan1996local} for further applications of such bias-corrected estimators. 
More specifically, the asymptotic biases in the kernel smoother and $k$-NN are corrected by local polynomial regression~\citep[LPoR;][]{fan1996local,tsybakov2009introduction} and multiscale $k$-NN~\citep[MS-$k$-NN;][]{okuno2020extrapolation}, respectively, where they are based on a similar idea: LPoR and MS-$k$-NN compute local regression~\citep{loader2006local} around the query, fitting a parametric model of the asymptotic bias, and extrapolating it towards the query for the bias to be eradicated.

While LPoR and MS-$k$-NN are based on the similar idea, they are different in their response variable and explanatory variable (Table~\ref{tab:methods}). Local regression function $u:\mathbb{R}^d \to \mathbb{R}$ in LPoR models $\eta(X)$ locally around the query $X_*$; $\eta(X)$ is estimated by $u(Z)$ with $Z:=X-X_*$ at each $X_*$, and it predicts the raw outcome $Y_i$ by estimating its label probability from the difference of covariates $Z_i:=X_{i}-X_* \in \mathbb{R}^d$ between the $i$-th sample $X_i$ and the query $X_*$.
On the other hand, $v:\mathbb{R}_{\ge 0} \to \mathbb{R}$ in MS-$k$-NN predicts the $k$-NN estimator $\hat{\eta}_{k}^{(k\text{NN})}(X_*)$ at each $X_*$ from the radial distance $r_k:=\|Z_{(k)}\|_2 \in \mathbb{R}_{\ge 0}$, where $(k)$ denotes the index of $k$-th nearest sample to the query; thus, $r_1 \le r_2 \le \cdots$ are ordered distances, and $\hat{\eta}_{k}^{(k\text{NN})}(X_*)$ is the average of the responses $Y_{(1)},\ldots, Y_{(k)}$.
Despite their theoretical optimality, they both rely on asymptotic theory, meaning that they consider the limit of the infinite number of observations for training. They do not perform as well as we expect in practical, real-world situations. 

\begin{table}[!ht]
\centering
\caption{Local regression for bias correction, used in LPoR, MS-$k$-NN and the proposed LRR. }  \label{tab:methods}
\begin{tabular}{lccc}
& Response var. & Explanatory var. & Regression func. \\
\midrule
Local Poly. Reg.~(LPoR) & $Y_i$ & $Z_i=X_i-X_*$ & $u:\mathbb{R}^d \to \mathbb{R}$ \\
Multiscale $k$-NN~(MS-$k$-NN) & $\hat{\eta}_k^{(k\text{NN})}(X_*)$ & $r_k=\|Z_{(k)}\|_2$ & $v:\mathbb{R}_{\ge 0} \to \mathbb{R}$ \\
\rowcolor{gray!10}
\textbf{Proposed} LRR  & $Y_{(i)}$ & $r_i = \|Z_{(i)}\|_2$  & $f:\mathbb{R}_{\ge 0} \to \mathbb{R}$  \\
\end{tabular}
\end{table}

In these circumstances, we provide the following contributions:

\begin{itemize} 
\item (\textbf{Methodology}) In Section~\ref{sec:LRR}, we propose a \emph{local radial regression~(LRR)} and its logistic regression variant called \emph{local radial logistic regression~(LRLR)}, by combining the advantages of LPoR and MS-$k$-NN. LRR, which includes LRLR as a special case, predicts (i)~the independent raw outcomes $Y_{(i)}$ by estimating its label probability directly as well as LPoR, using (ii) the regression function $f:\mathbb{R}_{\ge 0}\to\mathbb{R}$ of the radial distance $r_i=\|X_{(i)}-X_*\|_2 \in \mathbb{R}_{\ge 0}$ as well as MS-$k$-NN. 
The proposed LRR is expected to be trained with fewer observations for the following reasons:
(i') predicting independent raw outcomes $Y_{(i)}$ is easier than predicting mutually-dependent $k$-NN estimators as we can ignore the covariance of the outcomes, and 
(ii') training the regression function $f$ of $1$-dimensional radial component $r_i$ is easier as it includes fewer parameters than that of the $d$-dimensional raw covariate $Z_i$. 
For predicting $Y$ at $Z=0$, we do not need the ``full model'' with all the variables in $Z$ but the model with only $r=\|Z\|_2$ suffices; thus, the proposed LRR is preferable according to the principle of simplicity, i.e., Occam's Razor.
Note that most of our argument applies to a real-valued response $Y\in\mathbb{R}$ with $\eta(X)=\mathbb{E}(Y\mid X)$, but we focus on the binary response $Y\in\{0,1\}$ in this paper.

\item (\textbf{Theory}) In Section~\ref{subsec:improved_rate}, we prove the convergence rate of the point-wise $L^2$-risk for the proposed LRR. The obtained rate is compatible with those of LPoR and MS-$k$-NN. 

\item (\textbf{Experiments}) In Section~\ref{sec:experiments}, our numerical experiment shows that LRLR outperforms LPoR and MS-$k$-NN. 
In Section~\ref{sec:application}, we also apply LRLR to real-world datasets of daily stock indices called S\&P 500, S\&P/TSX, EURO STOXX 50, FTSE 100, DAX, CAC 40, TOPIX, and Hang Seng, by following \citet{nakagawa2018stock}. 
We classify whether month-end closing prices of target months are going up or down by LRLR, by comparing to the previous fluctuation patterns of stock indices. 
The predictive classification accuracy is improved from existing methods.  
\end{itemize}

This paper is partially based on a short paper of ours~\citep{Cao2021multiscale}, published in the proceedings of an annual domestic meeting of the Japanese Society for Artificial Intelligence~(JSAI).

\paragraph{Organization of this paper:} 
Background of this study, including the descriptions of the existing kernel smoother, $k$-nearest neighbor ~($k$-NN), local polynomial regression~(LPoR), and multiscale $k$-NN~(MS-$k$-NN), is explained in Section~\ref{sec:background}. 
The proposed local radial regression~(LRR), local radial logistic regression~(LRLR), and related discussions are described in Section~\ref{sec:LRR}. 
Numerical experiments on synthetic datasets are conducted in Section~\ref{sec:experiments}, application to real-world datasets is shown in Section~\ref{sec:application}, and we conclude this paper in Section~\ref{sec:conclusion}.

\section{Background}
\label{sec:background}

In this section, we explain the background of this study. 
Problem setting is described in Section~\ref{subsec:problem_setting}, and notation is shown in Section~\ref{subsec:notation}. 
Kernel smoother and $k$-NN are described in Section~\ref{subsec:KS_and_kNN}, local polynomial regression and multiscale $k$-NN correcting the asymptotic biases of the kernel smoother and $k$-NN are shown in Section~\ref{subsec:local_polynomial} and \ref{subsec:multiscale_kNN}, respectively. 

\subsection{Problem Setting}
\label{subsec:problem_setting}
Let $n \in \mathbb{N}$ be the sample size and $d\in \mathbb{N}$ be the dimensionality of the covariate.
Our dataset $\mathcal{D}_n$ consists of $(X_i,Y_i) \in \mathbb{R}^d \times \{0,1\}$ for $i=1,2,\ldots,n$, that are i.i.d.~with a distribution $\mathbb{Q}$. 
$X \in \mathbb{R}^d$ is called a \emph{covariate} and $Y \in \{0,1\}$ is called a \emph{label} (or more generally, \emph{outcome}), where $\eta(X)=\mathbb{P}(Y=1 \mid X)$ is also called label probability. 
The goal of supervised classification is to learn a classifier $\hat{g}_n:\mathbb{R}^d \to \{0,1\}$ using the observation $\mathcal{D}_n$, so that $\hat{g}_n(X)$ predicts $Y \in \{0,1\}$ for $(X,Y) \sim \mathbb{Q}$. 
In particular, this paper considers a plug-in type classifier
\begin{align}
    \hat{g}_n(X) := \mathbb{1}(\hat{\eta}_n(X) \ge 1/2)
    \label{eq:plug-in-classifier}
\end{align}
equipped with the estimated label probability by regression function $\hat{\eta}_n:\mathbb{R}^d \to \mathbb{R}$~(see, e.g., \citet{tsybakov2009introduction}, \citet{samworth2012optimal}, and \citet{okuno2020extrapolation}). 
A classifier $g(X):=\mathbb{1}(\eta(X) \ge 1/2)$ using the ground-truth label probability $\eta(X)$ is called Bayes-optimal classifier~\citep{devroye1996probabilistic}. 
Herein, this paper discusses estimating the regression function $\hat{\eta}_n$, so as to approximate the ground-truth $\eta$.

\subsection{Notation}
\label{subsec:notation}
In the remainder of this paper, we employ the following symbols. 
Let $[n]$ be the set $\{1,2,\ldots,n\}$ for any positive integer $n \in \mathbb{N}$. 
$\|X\|_2=(\sum_{i=1}^{d}x_i^2)^{1/2},\|X\|_{\infty}=\sup_{i \in [d]}|x_i|$ for a vector $X=(x_1,x_2,\ldots,x_d) \in \mathbb{R}^d$. 
$\lfloor \beta \rfloor:=\max\{\beta' \in \mathbb{N} \mid \beta' < \beta \}$ for $\beta>0$, i.e., 
the maximal integer that is strictly smaller than $\beta$; for example, $\lfloor 3.4 \rfloor = 3$, $\lfloor 3 \rfloor = 2$, so this is slightly different from the typical definition of the floor function with  $\lfloor 3 \rfloor = 3$. 
Given a query $X_* \in \mathbb{R}^d$, 
indices $1,2,\ldots,n$ of the observed covariates $X_1,X_2,\ldots,X_n \in \mathbb{R}^d$ are rearranged to $(1),(2),\ldots,(n)$ so that distances are ordered as $    \|X_{(1)}-X_*\|_2
\le 
\|X_{(2)}-X_*\|_2 
\le \cdots \le 
\|X_{(n)}-X_*\|_2$. 
Then we define, 
\begin{align}
    r_k := \|X_{(k)}-X_*\|_2
    \qquad 
    \label{eq:radius}
\end{align}
for $k \in [n]$. $d$-dimensional ball of radius $r$ is defined as $B_r(X_*):=\{X \in \mathbb{R}^d \mid \|X-X_*\|_2 \le r\} \subset \mathbb{R}^d$. 
$(X,Y) \in \mathbb{R}^d \times \{0,1\}$ denotes a pair of random variables following a distribution $\mathbb{Q}$ throughout this paper, and the conditional label probability with given $X$ is denoted by $\eta(X):=\mathbb{P}(Y=1 \mid X)$. 
$\mathbb{R}_{\ge 0},\mathbb{R}_+$ denote sets of non-negative and positive real values, respectively.

\subsection[Kernel Smoother and k-NN]{Kernel Smoother and $k$-NN}
\label{subsec:KS_and_kNN}

For estimating the label probability $\eta$, a kernel smoother ~\citep[KS;][]{nadaraya1964estimating,tsybakov2009introduction} and $k$-nearest neighbour estimator~\citep[$k$-NN;][]{fix1951nonparametric,cover1967nearest} are defined by
\begin{align}
    \hat{\eta}_{h}^{(\text{KS})}(X_*)
    &:=
    \argmin_{\theta \in \mathbb{R}} \sum_{i=1}^{n}
    \mathbb{1}(r_i \le h)
    \{Y_{(i)}-\theta\}^2
    =
    \frac{1}{\max\{i:r_i \le h\}}
    \sum_{i:r_i \le h} Y_{(i)}, \label{eq:kernel_smoother}\\
    \hat{\eta}_{k}^{(k\text{NN})}(X_*)
    &:=
    \frac{1}{k}\sum_{i \le k} Y_{(i)}, 
    \label{eq:kNN}
\end{align}
respectively, where $h>0$ and $k \in \mathbb{N}$ are hyper-parameters. 
$h$ is called bandwidth. 
Both of these two estimators take average of the observed labels $\{Y_{(i)}\}$, whose observed covariates $\{X_{(i)}\}$ are inside the balls $B_h(X_*)$ and $B_{r_k}(X_*)$, respectively; 
these two estimators are compatible if $h=r_k$. 
While the boxcar kernel $\mathbb{1}(r_i \le h)$ in the kernel smoother~(\ref{eq:kernel_smoother}) can be replaced with other types of kernel function~(see \citet{tsybakov2009introduction}), this paper considers only the boxcar case for simplicity.

In theory, the convergence rate of the plug-in classifier (\ref{eq:plug-in-classifier}) equipped with the kernel smoother is discussed in \citet{hall2005bandwidth}, while the convergence rate for classifier using the $k$-NN estimator is also proved in \citet{chaudhuri2014rates}. 
These two plug-in classifiers attain the same optimal rate if the ground-truth function $\eta(X)=\mathbb{P}(Y=1 \mid X)$ is Lipschitz but not differentiable. 
If the ground-truth function $\eta$ is highly-smooth (e.g., $\eta$ is twice-differentiable, so the target function class is more restricted), the theoretical efficiency limit for estimating $\eta$ is improved, and asymptotic higher-order biases become important in these estimators. 
See Appendix~\ref{app:asymptotic_biases} for the summary. 
The following local polynomial regression and multiscale $k$-NN correct these biases to attain the optimal rate, regardless of the smoothness.

\subsection{Local Polynomial Regression}
\label{subsec:local_polynomial}

Local polynomial regression estimator~\citep[LPoR;][]{fan1996local,tsybakov2009introduction} is defined by replacing the single real-valued parameter $\theta \in \mathbb{R}$ in (\ref{eq:kernel_smoother}) with a polynomial function $u:\mathbb{R}^d \to \mathbb{R}$ of degree $q$: 
\begin{align}
    \hat{\eta}_h^{(\text{LPoR})}(X_*)
    :=
    \hat{u}(\bs 0),
    \quad 
    \hat{u}
    :=
    \argmin_{\text{polynomial} \: u:\mathbb{R}^d \to \mathbb{R}}
    \sum_{i=1}^{n} \mathbb{1}(r_i \le h)
    \{
        Y_{(i)} - u(X_{(i)}-X_*)
    \}^2.
    \label{eq:local_polynomial}
\end{align}
By specifying $q=0$, i.e., $u(X)=\theta\in \mathbb{R}$ for any $X \in \mathbb{R}^d$, LPoR reduces to the kernel smoother~(\ref{eq:kernel_smoother}).

Herein, assume that the ground-truth label probability $\eta(X)=\mathbb{P}(Y=1 \mid X)$ is $q$-times continuously differentiable for some $q \in \mathbb{N}$. 
Intuitively speaking, the regression function $u:\mathbb{R}^d \to \mathbb{R}$ of degree $q$ learns the asymptotic higher-order bias around the query $X_*$, and eradicate the bias by extrapolating it to the query $X_*$ itself. 
See Appendix~\ref{app:asymptotic_biases} for the summary. 
Further assuming that the $q$-th derivatives of $\eta$ are no more differentiable, the plug-in classifier~(\ref{eq:plug-in-classifier}) equipped with the LPoR estimator $\hat{\eta}_h^{(\text{LPoR})}$ using the polynomial of degree $q$ is proved to be minimax optimal~\citep{samworth2012optimal}, by specifying an appropriate bandwidth $h=h_n \to 0$.

Besides the asymptotic optimality of LPoR, 
the polynomial function $u$ of the raw covariate $Z_i=X_i-X_* \in \mathbb{R}^d$ includes $1+d+d^2+\cdots+d^q$ parameters to be estimated: a larger number of observations is needed for training $u$ compared to the regression function of $1$-dimensional radial distance $r_i=\|Z_{(i)}\|_2 \in \mathbb{R}_{\ge 0}$. 
The regression function of the radial distance, used in the following multiscale $k$-NN and our proposed LRR, includes only $1+q$ parameters.

\subsection[Multiscale k-NN]{Multiscale $k$-NN}
\label{subsec:multiscale_kNN}

While LPoR corrects the asymptotic higher-order bias of the kernel smoother, multiscale $k$-NN~\citep[MS-$k$-NN;][]{okuno2020extrapolation} shown below is designed to eradicate the bias in $k$-NN estimator $\hat{\eta}^{(k\text{NN})}_k$: 

\begin{align}
    \hat{\eta}_{k_1,k_2,\ldots,k_J}^{(\text{MS}k\text{NN})}(X_*)
    :=
    \hat{v}(0),
    \quad 
    \hat{v}
    :=
    \argmin_{\text{polynomial} \: v:\mathbb{R}_{\ge 0} \to \mathbb{R}}
    \sum_{j=1}^{J}
    \{
        \hat{\eta}_{k_j}^{(k\text{NN})}(X_*)
        -
        v(\underbrace{\|X_{(k_j)}-X_*\|_2}_{=r_{k_j} \in \mathbb{R}_{\ge 0}})
    \}^2.
\label{eq:msknn}
\end{align}
$1 \le k_1 < k_2 < \cdots < k_J \le n$ with $J \in \mathbb{N}$ are user-specified parameters. 
By specifying the regression function $v$ and the parameters $\bs k=(k_{1},k_{2},\ldots,k_{J})$ appropriately, 
extrapolating the trained $\hat{v}(r)$ towards $r=0$ yields an imaginary $0$-NN whose asymptotic bias is eradicated. See Appendix~\ref{app:asymptotic_biases} for the summary.
The plug-in classifier~(\ref{eq:plug-in-classifier}) equipped with the MS-$k$-NN estimator $\hat{\eta}^{(\text{MS}k\text{NN})}_{k_1,k_2,\ldots,k_J}$ attains the minimax optimal rate, same as the above LPoR. See \citet{okuno2020extrapolation} Theorem~2 for details.

As for the local regression function $v$, we may employ a logistic function
\begin{align}
    v(r):=\sigma(\theta_0+\theta_1 r+\theta_2 r^2+\cdots+\theta_q r^q):\mathbb{R}_{\ge 0} \to [0,1]
    \label{eq:logistic_function}
\end{align}
using a sigmoid function $\sigma(z):=(1+\exp(-z))^{-1}$. 
$\theta=(\theta_0,\theta_1,\ldots,\theta_q) \in \mathbb{R}^{1+q}$ is a parameter vector to be estimated. 
In our numerical experiments, we distinguish MS-$k$-NN (poly.) and MS-$k$-NN (logi.), where the former employs the polynomial function $v(r)=\theta_0+\theta_1 r+\theta_2 r^2+\cdots+\theta_q r^q$ while the latter employs the logistic function~(\ref{eq:logistic_function}).

In this MS-$k$-NN, the local regression function $v:\mathbb{R}_{\ge 0} \to \mathbb{R}$ is trained as if the target variables (i.e., $k$-NN estimators) were mutually independent. 
$k$-NN estimators are in fact dependent as they share a part of the labels, i.e., $\hat{\eta}^{(k\text{NN})}_k$ and $\hat{\eta}^{(k\text{NN})}_{k'}$ share $Y_{(1)},Y_{(2)},\ldots,Y_{(\min\{k,k'\})}$: the loss function~(\ref{eq:msknn}) should be weighted by the inverse matrix of the covariance matrix of outcomes, or independent outcomes should be considered as well as LPoR and our proposed LRR, to obtain a better estimator~(see, e.g., \citet{KariyaTakeaki2004GlsT} Theorem~2.1 for the efficiency of general least squares predicting dependent outcomes).

\section{Local Radial Regression}
\label{sec:LRR}
As explained in the previous sections, LPoR and MS-$k$-NN rely on the asymptotic theory, meaning that they consider the limit of the infinite number of observations for training. 
They do not perform as well as we expect in finite real-world situations, whereby they have substantial room for improvement. For computing bias-corrected estimator with fewer observations, we propose a local radial regression~(LRR) and local radial logistic regression in Section~\ref{subsec:LRR}. 
The relation to the existing LPoR and MS-$k$-NN is described in Section~\ref{subsec:relation_to_LPoR_MSkNN}. 
The improved convergence rate is described in Section~\ref{subsec:improved_rate}

\subsection{Local Radial Regression (LRR) and Local Radial Logistic Regression (LRLR)}
\label{subsec:LRR}

We propose a \emph{local radial regression~(LRR)} as 
\begin{align}
    \hat{\eta}^{(\text{LRR})}(X_*)
    :=
    \hat{f}(0),
    \quad 
    \hat{f}:=\argmin_{f:\mathbb{R}_{\ge 0}\to\mathbb{R}}
    \sum_{i=1}^{n}
    w(r_i)\ell\left( 
        Y_{(i)},f(r_i)
    \right),
    \label{eq:LRR}
\end{align}
where $w(r):\mathbb{R}_{\ge 0} \to \mathbb{R}_{\ge 0}$ denotes a user-specified weight function, 
$f:\mathbb{R}_{\ge 0} \to \mathbb{R}$ denotes a user-specified parametric regression function such as the polynomial $f(r) = \theta_0 + \theta_1 r + \theta_2 r^2 + \cdots + \theta_q r^q$ and 
$\ell(\cdot, \cdot):\mathcal{Y} \times \mathcal{Y} \to \mathbb{R}_{\ge 0}$ denotes a loss function for some set $\mathcal{Y} \subset \mathbb{R}$ such as 
the squared loss $\ell(y,y'):=(y-y')^2$. 
LRR predicts the raw outcome $Y_{(i)}$ directly as well as LPoR, while 
LRR uses the radial distance $r_{i}:=\|X_{(i)}-X_*\|_2$ of the covariate $X_{(i)}$ as well as MS-$k$-NN.
Therefore, the proposed LRR combines the advantages of LPoR and MS-$k$-NN.

Note that discussions on the weight function for local (non-radial) regression can be found in a range of existing studies. For example, the weighting for robust local regression has been well developed~(see, e.g.,  \citet{cleveland1979robust,cleveland1981lowess} and \citet{cleveland1988locally}). However, the weight function concerning the radial distance $r_k$ is slightly out of scope of the main interest in local regression studies; in our numerical experiments, we specify $w(r)=1$ or $w(r)=1/r$.

Same as MS-$k$-NN, we may employ the logistic regression function (\ref{eq:logistic_function}) for $f$; training the logistic regression function by minimizing the logistic loss~(i.e., negative log-likelihood for the binomial distribution) $\ell(y,y'):=y \log y'+(1-y)\log(1-y'),\mathcal{Y}=[0,1]$ is considered as a special case of LRR. We call it \emph{local radial logistic regression~(LRLR)}.

\subsection[Relation to LPoR and Multiscale k-NN]{Relation to LPoR and Multiscale $k$-NN}
\label{subsec:relation_to_LPoR_MSkNN}

In this section, we provide an intuitive explanation for the relations among LRR, LPoR, and MS-$k$-NN. Please see Section \ref{subsec:improved_rate} for more rigorous theory on LRR. 
Here, for simplicity, we employ the squared loss $\ell(y,y')=(y-y')^2$, a fixed query $X_*$, a constant bandwidth $h>0$, $J=J(h):=\max_{j\in [n]}\{j : r_{j} \le h\}$, $k_j=j$ for $j \in [n]$ and a weight function $w(r)=\mathbb{1}(r \le h)$. 
Then, the local regression function trained in LPoR, MS-$k$-NN, and the proposed LRR are 
\begin{align*}
    \hat{u}&:=\argmin_{u}
    \left\{ 
    L_{1,J}(u):=
    \frac{1}{J}\sum_{j=1}^{J}\{Y_{(j)}-u(X_{(j)}-X_*)\}^2
    \right\}, \\
    \hat{v}&:=\argmin_{v}
    \left\{ 
    L_{2,J}(v):=
    \frac{1}{J}\sum_{j=1}^{J}\{\hat{\eta}_{j}^{(k\text{NN})}(X_*)-v(r_{j})\}^2
    \right\}, \\
    \hat{f}&:=\argmin_{f}
    \left\{
    L_{3,J}(f):=
    \frac{1}{J}
    \sum_{j=1}^{J}\{Y_{(j)}-f(r_{j})\}^2
    \right\},
\end{align*}
respectively. 
The law of large numbers proves 
\begin{align}
    L_{1,J}(u)
&\overset{\text{in probability}}{\to}
    \mathbb{E}_{Z}\left( 
        \left\{\mathbb{P}(Y = 1 \mid X-X_*=Z) -u(Z)\right\}^2
    \right)+C_1, \label{eq:LLN_LPoR} \\
    L_{2,J}(v)
&\overset{\text{in probability}}{\to}
    \mathbb{E}_{r}\left( 
        \left\{\mathbb{P}(Y = 1 \mid \|X-X_*\|_2 \le r) -v(r)\right\}^2
    \right)+C_2, \label{eq:LLN_MSkNN} \\
    L_{3,J}(f)
&\overset{\text{in probability}}{\to}
    \mathbb{E}_{r}\left( 
        \left\{\mathbb{P}(Y = 1 \mid \|X-X_*\|_2 = r) -f(r)\right\}^2
    \right)+C_3 \label{eq:LLN_LRR}
\end{align}
as $n \to \infty$ (whereby $J \to \infty$) for some $C_1,C_2,C_3 \ge 0$. Here $\mathbb{E}_Z$ and $\mathbb{E}_r$ denotes the expectation with respect to $Z = X' - X_*$ and $r = \|Z\|_2$, respectively, by considering that $X'$ is drawn independently from the same distribution of $X$ but conditioned as $\|X' - X_*\| \le h$.
The limit of $L_{2,J}(v)$ is particularly obtained by \citet{chaudhuri2014rates} Lemma~9 proving that $\hat{\eta}^{(k\text{NN})}_{k_J}(X_*)$ is asymptotically compatible with $\mathbb{P}(Y=1 \mid X \in B_h(X_*))$. 
These limits of the loss functions indicate what function each local regression is estimating. 

Firstly, LPoR is compatible with LRR in some simple cases, regardless of whether we consider the asymptotic case or not: LPoR with $u(Z)=f(\|Z\|_2)$ reduces to LRR in this simple setting. This restriction on the regression function $u$ for LRR yields better empirical performance of LRR than general existing LPoR. Also see our numerical experiments in Section~\ref{sec:experiments}.  

Regarding the MS-$k$-NN, its regression function $\hat{v}$ approximates the label probability over the ball $B_r(X_*)$, while $\hat{f}$ in the proposed LRR approximates the probability over the ball surface $\partial B_r(X_*)$. 
These two estimators $\hat{v}$ and $\hat{f}$ are different but compatible in the limit $r \searrow 0$. 
Despite their asymptotic compatibility, these two estimators are different in the finite real-world setting, where the function $\hat{f}$ in LRLR is expected to be estimated more efficiently than $\hat{v}$ in MS-$k$-NN due to the dependency issue described at the last of Section~\ref{subsec:multiscale_kNN}.

\subsection{Convergence Rate Analysis}
\label{subsec:improved_rate}

In this section, we show a convergence rate of the point-wise $L^2$-risk for LRR, at a fixed query $X_* \in \mathbb{R}^d$. We first describe the detailed settings and conditions, to prove the convergence rate in Theorem~\ref{theo:convergence_rate_LRLR}.

For simplicity of the proof, we employ a specific form of the weights and the loss function:
\begin{align}
    w_n(r)
    :=
    \begin{cases}
    1/N_n & (r \le \tilde{r}_n) \\
    0 & (r>\tilde{r}_n)
    \end{cases},
    \quad 
    \ell(a,b):=(a-b)^2,
    \label{eq:specific_weights}
\end{align}
with a positive sequence $\{\tilde{r}_n\}_{n \in \mathbb{N}}$.
Herein, 
\[
    N_n:=\sum_{i=1}^{n} \mathbbm{1}(\|X_i-X_*\|_2 \le \tilde{r}_n)
\]
denotes the number of positive weights (i.e., the number of observations used for local regression) and $\overline{N}_n:=\mathbb{E}(N_n)$ denotes its expectation.
We assume that the sequence $\{\tilde{r}_n\}_{n \in \mathbb{N}}$ satisfies
\begin{align}
    \tilde{r}_n \searrow 0
    \quad \text{and} \quad  
    n^{1-\nu} \tilde{r}_n^d \to \infty
    \quad (n \to \infty)
    \label{eq:rate_rn}
\end{align}
for some $\nu > 0$, indicating that $\overline{N}_n \ge c n\tilde{r}_n^d \to \infty$ (for some $c>0$). Note that, while (\ref{eq:rate_rn}) is a slightly stronger condition than that of the bandwidth in standard nonparametric theories~($h_n \searrow 0,n h_n^d \to \infty$; see, e.g., \citet{tsybakov2009introduction}), the optimal radius $\tilde{r}_n$ obtained by Theorem~\ref{theo:convergence_rate_LRLR} still satisfies (\ref{eq:rate_rn}). 

We employ a set of $2\omega$-th degree polynomials with even-degree terms
\begin{align}
    \setF
    =
    \setF(\omega)
    :=
    \left\{ f(r)
    :=
    \theta_0+\sum_{c=1}^{\omega} \theta_c r^{2c} \mid \theta_0,\theta_1,\ldots,\theta_{\omega} \in \mathbb{R}\right\}
    \label{eq:polynomial}
\end{align}
for $\omega \in \mathbb{N}$. 
$\omega$ will be specified by $\omega=\lfloor \beta/2 \rfloor$ under the assumptions that the distributions of the label and covariate are $\beta$-H{\"o}lder (see Condition (C-2) for details). 
The settings (\ref{eq:specific_weights}) and (\ref{eq:polynomial}) are compatible with the existing theories of MS-$k$-NN~\citep{okuno2020extrapolation}.


For proving the convergence rate of LRR, we assume the following conditions (C-1)--(C-3) on the density function $\mu$ of the covariate $X$ and the label probability $\eta(X)=\mathbb{P}(Y=1 \mid X)$, 
with the fixed query $X_* \in \mathbb{R}^{d}$ and a fixed constant $R \in (0,1)$.

\begin{enumerate}[{(C-1)}]
\item Over the ball $B_R(X_*)$, $\mu$ is lower-bounded by a positive constant $l>0$. 

\item Over the ball $B_R(X_*)$, $\mu(X)$ and $(\eta\mu)(X) := \eta(X)\mu(X)$ are $\beta$-H{\"o}lder, i.e., there exist a polynomial $b_*^{(q)}:\mathbb{R}^d \to \mathbb{R}$ of degree $q:=\lfloor \beta \rfloor$ satisfying $b_*^{(q)}(\bs 0)=0$ and $L_{\mu} \in (0,\infty)$ such that
\begin{align} 
\mu(X) = \mu(X_*) + b_*^{(q)}(X-X_*) + \varepsilon_{\mu}(X),
\qquad 
|\varepsilon_{\mu}(X)| \le L_{\mu}\|X-X_*\|_2^{\beta}
\quad \: (\forall X \in B_R(X_*)).
\label{eq:taylor_mu}
\end{align}
(\ref{eq:taylor_mu}) is analogous to Taylor expansion, and a similar expression can be obtained for $\eta\mu$. 

\item There exist $c,C,\tau>0$ and $\varphi \in (0,1)$ such that
\[
\mathbb{P}(\zeta_n \ge \varphi N_n) \le C\exp(-c n^{\tau})
\]
holds for 
\[
    \tilde{R}_n := \left(\begin{array}{cccc}
    r_{1}^2 & r_{1}^4 & \cdots & r_{1}^{2 \lfloor \beta/2 \rfloor} \\
    r_{2}^2 & r_{2}^4 & \cdots & r_{2}^{2 \lfloor \beta/2 \rfloor} \\
    \vdots & \vdots & \ldots & \vdots \\
    r_{N_n}^2 & r_{N_n}^4 & \cdots & r_{N_n}^{2 \lfloor \beta/2 \rfloor}
    \end{array}\right),
    \,
    \proj_{\tilde{R}_n}:=\tilde{R}_n(\tilde{R}_n^{\top}\tilde{R}_n)^{-1}\tilde{R}_n^{\top},
    \, 
    1_{N_n}:=(\underbrace{1,1,\ldots,1}_{N_n})^\top,
    \, 
    \zeta_n:=\langle 1_{N_n}, \proj_{\tilde{R}_n} 1_{N_n}\rangle.
\]
\end{enumerate}

The above Conditions (C-1)--(C-3) are compatible with \citet{okuno2020extrapolation}. 
Therein, (C-1) and (C-2) are called strong density assumption~(SDA) and ($\beta, L_\mu, \mathbb{R}^d$)-H{\"o}lder class of functions, respectively, with reference to \citet{audibert2007learning}. 
(C-3) is a modification of the condition (C-3) in \citet{okuno2020extrapolation} Section 4.3. 
As the LRR estimator is an intercept of the polynomial regression, a simple matrix algebra indicates that the LRR estimator is compatible with the weighted average of the outcomes $\sum_{i=1}^{N_n} \rho_{n,i}Y_{(i)}$ for some weights $\{\rho_{n,i}\}_{i=1}^{N_n} \subset \mathbb{R}$ defined in (\ref{eq:zeta_n}); then, the condition (C-3) is needed to prevent the weights $\{\rho_{n,i}\}_{i=1}^{N_n}$ from being diverged. 
We provide Example~\ref{ex:C-3} satisfying the condition (C-3), with the proof shown in Appendix~\ref{proof:ex:C-3}.  

\begin{ex}
\label{ex:C-3}
    Assuming that $\lfloor \beta/2 \rfloor=1$ and $X$ distributes uniformly (at least over the ball $B_R(X_*)$), (C-3) is satisfied with $\varphi=1-1/\{2(d+1)^2\}$. 
\end{ex}

\bigskip
To ensure a proper estimation of LRR, we consider an event
\[
    \event: \quad 
    N_n \ge 1+\lfloor \beta/2 \rfloor 
    \, \text{ and } \, 
    \zeta_n \le \varphi N_n
\] 
with $\beta,\zeta_n,\varphi$ specified in (C-2) and (C-3), and we redefine the LRR estimator 
\begin{align}
    \hat{\eta}_n(X_*) := 
    \begin{cases} 
    \hat{f}_n(0),
    \quad 
    \hat{f}_n := \argmin_{f \in \setF} \sum_{i=1}^{n} w_n(r_i) \{Y_{(i)}-f(r_i)\}^2 & (\text{if } \event)\\
    0 & (\text{if }\lnot \event)
    \end{cases}
    \label{eq:LRLR_specific}
\end{align}
with the negation $\lnot \event$. 
Then, the convergence rate of the $L^2$ risk is evaluated as follows. 

\begin{theorem}
\label{theo:convergence_rate_LRLR}
Assuming (C-1)--(C-3) with the fixed query $X_* \in \mathbb{R}^d$, the LRR estimator (\ref{eq:LRLR_specific}) equipped with $\omega=\lfloor \beta/2 \rfloor$ and (\ref{eq:specific_weights})--(\ref{eq:polynomial}) satisfies
\[
\mathbb{E}_{\setDn}(\{\eta(X_*)-\hat{\eta}_n(X_*)\}^2)
\le 
    C \: \max\left\{ 
     \, \tilde{r}_n^{2\beta}, \, \frac{1}{n \tilde{r}_n^d}
    \right\}
\]
for some $C>0$.
\end{theorem}

See Appendix~\ref{proof:theo:convergence_rate_LRLR} for the proof. 
The obtained upper-bound in Theorem~\ref{theo:convergence_rate_LRLR} is compatible with LPoR with $h=h_n=\tilde{r}_n$~\citep{tsybakov2009introduction}. 
Theorem~\ref{theo:convergence_rate_LRLR} immediately proves the optimal rate of the LRR as follows.

\begin{coro}
\label{coro:LRLR_optimality}
$\mathbb{E}_{\setDn}(\{\eta(X_*)-\hat{\eta}_n(X_*)\}^2) \le C n^{-2\beta/(d+2\beta)}$ with the optimal radius $\tilde{r}_n:=n^{-1/(d+2\beta)}$. 
\end{coro}

The obtained rate $n^{-2\beta/(d+2\beta)}$ is optimal in estimating $\beta$-H{\"o}lder functions~(see, e.g., \citet{tsybakov2009introduction}), and this rate is compatible with the rates for LPoR and MS-$k$-NN:
\[
    \mathbb{E}_{\setDn}(\{\eta(X_*)-\hat{\eta}_{h_n}^{(\text{LPoR})}\}^2) , \quad
    \mathbb{E}_{\setDn}(\{\eta(X_*)-\hat{\eta}_{k_{n,1},k_{n,2},\ldots,k_{n,J}}^{(\text{MS}k\text{NN})}\}^2) \quad 
    \le 
    \quad 
    C n^{-2\beta/(d+2\beta)}
\]
for some $C>0$, with the optimal bandwidth $h_n := n^{-1/(d+2\beta)}$ and $k_{n,j}:=\min\{k=1,2,\ldots,n \mid \|X_{(k)}-X_*\|_2 \ge \ell_j r_1\}$ ($j=1,2,\ldots,J$), $k_{n,1} := n^{2\beta/(2\beta+d)},r_1:=\|X_{(k_{n,1})}-X_*\|_2$ and $1=\ell_1<\ell_2<\ell_3<\cdots<\ell_J<\infty$. 
See \citet{tsybakov2009introduction} and \citet{okuno2020extrapolation} for details.

Note that, Corollary~\ref{coro:LRLR_optimality} proves the point-wise convergence at the fixed query $X_*$ in the sense of $L^2$-risk (i.e., regression), while \citet{okuno2020extrapolation} also provides the uniform convergence in the sense of misclassification error (i.e., classification) using the plug-in type classifier~(\ref{eq:plug-in-classifier}). 
See \citet{audibert2007learning} for a relation between the convergence rates of the $L^2$-risk and the misclassification error. 

\section{Experiments on Synthetic Datasets}
\label{sec:experiments}

In this section, we conduct numerical experiments on synthetic datasets. We compare the plug-in classifier~(\ref{eq:plug-in-classifier}) of the proposed LRLR to those of logistic regression, $k$-NN (which is also compatible with kernel smoother), LPoR, a logistic regression variant of LPoR~(LPoLR), and MS-$k$-NN.

\subsection{Experimental Setting}

\begin{itemize}
    \item \textbf{Datasets}: 
    Let $n_{\text{train}}=n_{\text{test}}=500$ be number of observations in training and test datasets, and let $d=3$. 
    For training sets, we randomly generate $X_i=(x_{ij}) \in \mathbb{R}^d$ with $x_{ij} \overset{\text{i.i.d.}}{\sim} U[-1,1]$, and $Y_i \mid X_i \overset{\text{i.i.d.}}{\sim} \text{Be}(\rho(\eta(X_i)+\varepsilon_i))$ with $\varepsilon_i \overset{\text{i.i.d.}}{\sim} N(0,0.05^2)$, 
    $\rho(z):=\argmin_{z' \in [0,1]}\|z-z'\|_2$ and 
    \begin{align}
        \eta(X_i):=15\prod_{j=1}^{d}\phi\left(2(x_{ij}-1/2)\right) + 15\prod_{j=1}^{d}\phi\left(2(x_{ij}+1/2)\right) 
        \in [0,1],
    \label{eq:groundtruth_synthetic}
    \end{align}
    for $i=1,2,\ldots,n_{\text{train}}$.
    $\text{Be}(\mu)$ denotes Bernoulli distribution with mean $\mu \in [0,1]$ and $\phi$ denotes the probability density function of univariate standard normal distribution. 
    For test sets, we randomly generate $X_i=(x_{ij})$ with $x_{ij}\overset{\text{i.i.d.}}{\sim} U[-0.7,0.7]$ and 
    $Y_i \mid X_i \overset{\text{i.i.d.}}{\sim} \text{Be}(\eta(X_i))$ for $i=1,2,\ldots,n_{\text{test}}$. 
    
    \item \textbf{Evaluation}: 
    we take average of the concordance rates of each plug-in classifier to (i) randomly generated test labels $Y_i^{(\text{test})} \in \{0,1\}$ and the (ii) optimal Bayes classifier $g(X_i^{\text{(test)}}) \in \{0,1\}$, over $200$ times experiments. 
    
    \item \textbf{Baselines}: 
    We compute $k$-NN for each $k=k_j$ and MS-$k$-NN (logi.) with $\bs k$ ($J=5$) 
    for $\bs k=(k_1,k_2,k_3,k_4,k_5)=(10,20,30,40,50)$. 
    We also compute LPoR and the logistic regression variant of LPoR (denoted by LPoLR, see \citet{lee2006conditional}) with bandwidth $h=0.4$.
    Kernel smoother is not listed in the baselines, as it is compatible with $k$-NN as explained in Section~\ref{subsec:KS_and_kNN}. 
    In addition to those methods described in the previous sections, we include
    the random prediction: random predictor outputs $1$ and $0$ with probability $1/2$ and $1/2$, respectively, regardless of the query $X_*$. 
    As a representative of parametric discrimination method, we used  the logistic regression.
    
    \item \textbf{Parametric regression function}: 
    LPoR employs the polynomial function of degree $2$. 
    Logistic regression, MS-$k$-NN (logi.), LPoLR, and LRLR employ the polynomial of degree $2$ with the sigmoid function. 
    In particular, MS-$k$-NN (logi.) in this section optimizes the regression function $v(r)$ by minimizing the logistic loss $\sum_{j=1}^{J}\ell_{\text{logi.}}(\hat{\eta}_{k_j}^{(k\text{NN})},v(r_{k_j}))$ defined with $\ell_{\text{logi.}}(\eta,v)=\eta \log v + (1-\eta) \log (1-v)$, but not the squared loss $\ell(\eta,v)=(\eta-v)^2$ used in (\ref{eq:msknn})
\end{itemize}

\subsection{Results} 

Experimental results are shown in Table~\ref{table:results_synthetic}. For both (i) and (ii), the proposed LRLR outperforms the remaining existing methods. In particular, LRLR with constant weight function $w(r)=1$ shows the best score, and subsequently, LRLR with the decreasing weight $w(r)=1/r$ is almost the second best. 
This classification problem is particularly difficult for extrapolation-based approaches and logistic regression, as the ground-truth function $\eta$ is bi-modal. 

The scatter plot matrix of estimators in an instance of the above experiment is also shown in Figure~\ref{fig:scatter}: we can see the detailed correlation between each pair of regression methods. For instance, we can examine the high correlation coefficient between $k$-NN and MS-$k$-NN (logi.). 
The logistic regression does not perform well as its expressive power is limited to approximate the bi-modal ground-truth function (\ref{eq:groundtruth_synthetic}). 
LPoR is hard to compute stably, as its regression function is not restricted to take value within $[0,1]$ and it uses raw $d$-dimensional covariates; LPoR and LPoLR tend to be unstable. 
LRLR is stably computed, and the correlation coefficient to the ground-truth is the highest among all the methods. 

Further experiments with more realistic settings (including the selection of parameters by cross-validation) are shown in the following Section~\ref{sec:application}.

\begin{table}[!ht]
\centering
\caption{Concordance rates to (i) randomly generated test labels, (ii) optimal Bayes classifier outputs, in $200$ times experiments. 
Average and standard error are listed. 
A higher score is better: the best and the second-best are bolded and underlined, respectively.}
\label{table:results_synthetic}
\scalebox{0.9}{
\begin{tabular}{llcc}
\toprule 
\multicolumn{2}{c}{Method} & (i) Test Labels & (ii) Optimal Bayes Clas. \\
\midrule 
Random & & $0.500 \pm 0.002$ & $0.500 \pm 0.002$ \\
Logistic & & $0.625 \pm 0.002$ & $0.732 \pm 0.003$ \\
\multirow{5}{*}{$k$-NN} & ($k=10$) & $0.698 \pm 0.002$ & $0.866 \pm 0.002$ \\
& ($k=20$) & $0.705 \pm 0.002$ & $0.888 \pm 0.002$ \\
& ($k=30$) & $0.706 \pm 0.002$ & $\second{0.894} \pm 0.002$ \\
& ($k=40$) & $0.705 \pm 0.002$ & $0.891 \pm 0.002$ \\
& ($k=50$) & $0.702 \pm 0.002$ & $0.883 \pm 0.002$ \\
MS-$k$-NN~(logi.) & & $0.569 \pm 0.002$ & $0.612 \pm 0.002$ \\
LPoR & ($h=0.4$) & $0.551 \pm 0.002$ & $0.603 \pm 0.002$ \\
LPoLR& ($h=0.4$) & $0.569 \pm 0.002$ & $0.629 \pm 0.002$ \\
\rowcolor{gray!10}
& ($w(r)=1$) & $\first{0.716} \pm 0.002$ & $\first{0.910} \pm 0.002$ \\
\rowcolor{gray!10}
\multirow{-2}{*}{LRLR}
& ($w(r)=1/r$) & $\second{0.707} \pm 0.002$ & $0.881 \pm 0.002$ \\
\bottomrule
\end{tabular}
}
\end{table}

\begin{figure}[!ht]
\centering 
\includegraphics[width=0.9\textwidth]{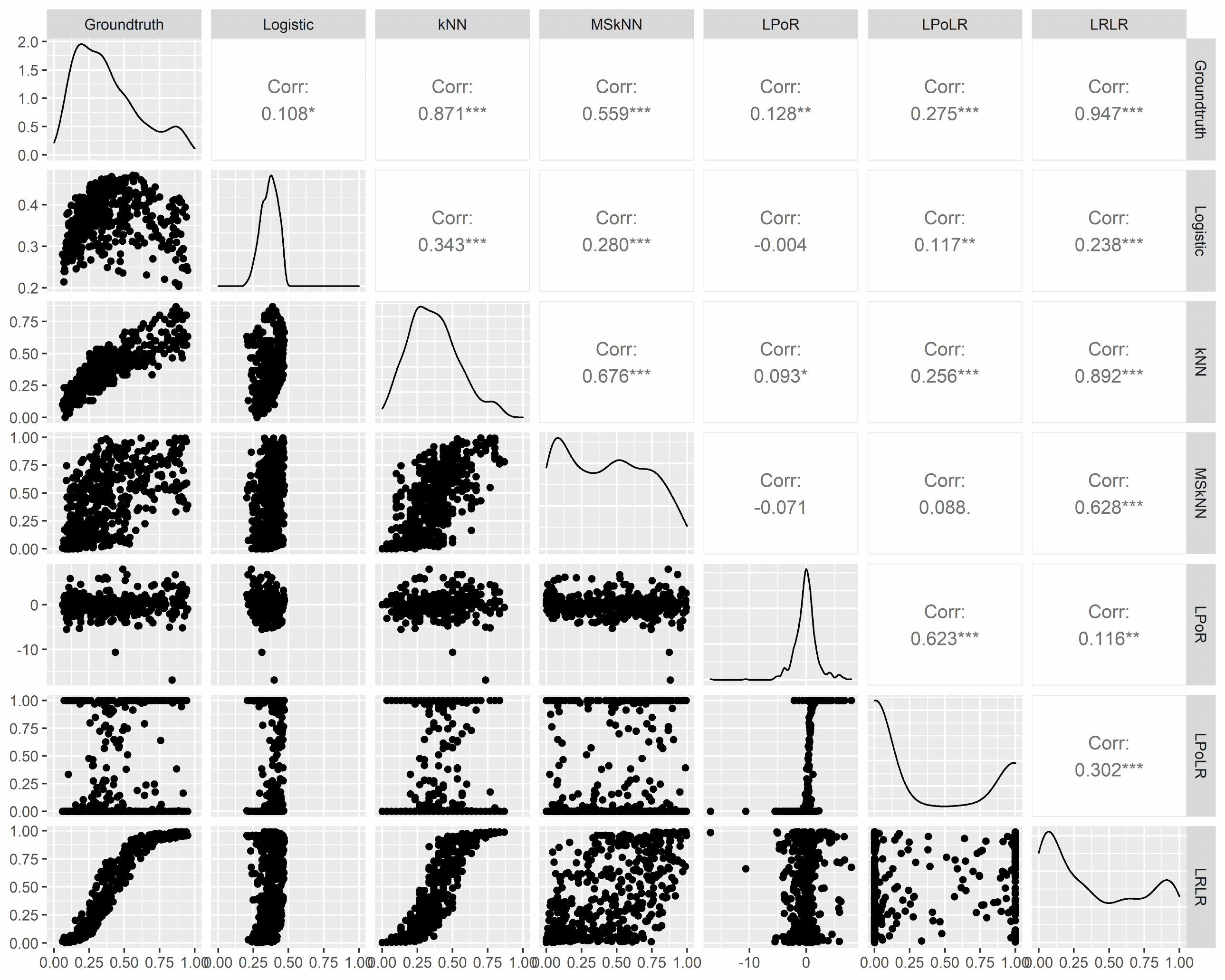}
\caption{Scatter plot matrix of predicted label probabilities: ground-truth $\eta$, logistic regression, $k$-NN with $k=30$, MS-$k$-NN with $\bs k=(10,20,30,40,50)$, LPoR with $h=0.4$, and LRLR with $w(r)=1/r$.}
\label{fig:scatter}
\end{figure}

\section{Application to Predictive Month-End Closing Price Classification}
\label{sec:application}

In this section, we conduct experiments on real-world datasets, that predict the rise or fall of eight major world stock indices over a period from 1989 to 2021. 
Experiments in this section basically follow \citet{nakagawa2018stock}, which conducts similar experiments with the stock indices before 2018.

\subsection{Overview}
\label{subsec:overview}

In our experiments, we consider classifying the month-end closing prices of eight major world stock indices, S\&P500~(U.S.), S\&P/TSX~(Canada), FTSE100~(U.K.), DAX~(German), CAC40~(France), Euro Stoxx 50~(EU), TOPIX~(Japan), and Hang Seng Index~(Hong Kong), obtained from Jan. 1989 to Oct. 2021. 
The covariate $X_i \in \mathbb{R}_+^{d_i}$ denotes the concatenation of daily closing prices in a month indexed by $i$, and the response $Y_i \in \{0,1\}$ denotes whether the closing price at the last day in the month $i+1$ rises or falls from the month $i$. 
Here $d_i$ denotes the number of days the market is open, which varies from month to month.
Given an observed covariate $X_* \in \mathbb{R}_+^{d_*}$ of a query month, we predictively classify whether the month-end closing price rises or not.  
See Section~\ref{subsec:experimental_setting} for more details, and also see Figure~\ref{fig:overview_stock_price_classification} for illustration of this classification problem.

\begin{figure}[!ht]
\centering
\includegraphics[width=0.8\textwidth]{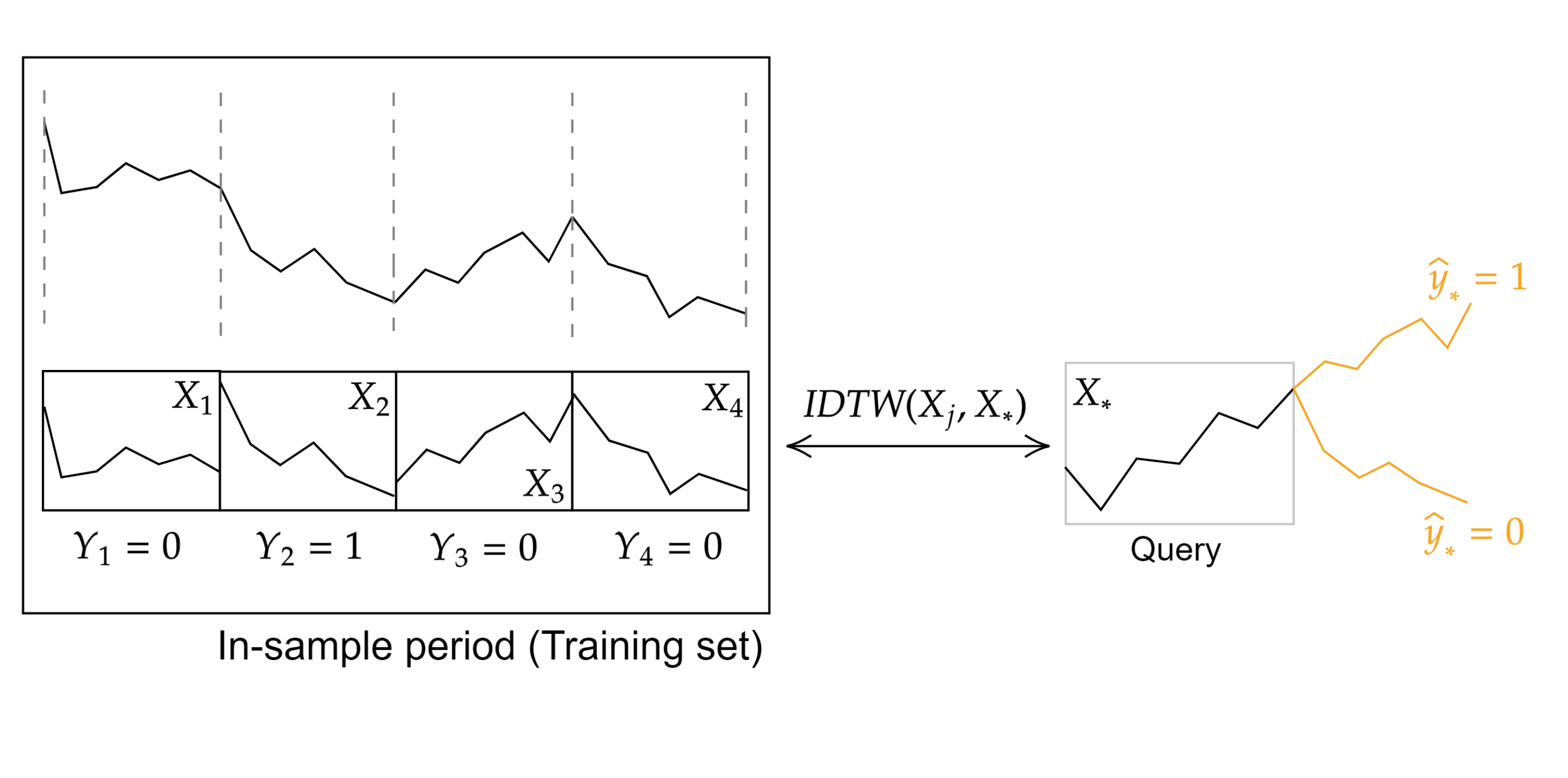}
\caption{Overview of the predictive month-end closing price classification. 
$X_i=(x_{i1},x_{i2},\ldots,x_{i d_{i}}) \in \mathbb{R}_+^{d_i}$ concatenates daily closing prices in the month $i$, and $Y_i \in \{0,1\}$ represents whether the month-end closing price of the month $i+1$ rises or falls from that of the month $i$. 
We classify the query $X_* \in \mathbb{R}_+^{d_*}$, whether the price would rise or fall in the next month.}
\label{fig:overview_stock_price_classification}
\end{figure}

An interesting point of these datasets is that each observed covariate $X_i$ has the individual dimensionality $d_i$, so we cannot compute simple distances (e.g., Euclidean distance) between  $X_i$ and $X_j$ with different dimensionalities $d_i \ne d_j$. 
To measure their discrepancy, we employ the dynamic time warping~(DTW; \citet{bellman1959adaptive} and \citet{Muller2007} Section 4); $\dtw(X_i,X_j)$ denotes the Euclidean distance between the best-aligned vectors, which may be computed by the dynamic programming algorithm.
For comparing daily stock prices, we first rescale $X_i$ and $X_j$ so that their first elements become one: $\tilde X_i := X_i/x_{i1}$ and  $\tilde X_j := X_j/x_{j1}$.
We then apply DTW to the rescaled vectors; this method is called indexing DTW~\citep[IDTW;][]{nakagawa2018stock}, denoted as
 $\idtw(X_i,X_j) := \dtw(\tilde X_i, \tilde X_j)$.
We replace the radial distance (\ref{eq:radius}) by 
\begin{align}
    r_i := \idtw(X_{(i)},X_*)
    \label{eq:idtw_radial}
\end{align}
satisfying $0 \le r_1 \le r_2 \le \cdots \le r_{n_{\text{train}}}$, yielding IDTW variants of $k$-NN, multiscale $k$-NN and LRLR. 
The simple logistic regression, LPoR, and LPoLR cannot be computed for these datasets, as their regression functions require covariates of the same dimensionality $d$.

\subsection{Experimental Setting}
\label{subsec:experimental_setting}

\begin{itemize}
\item \textbf{Datasets}: Daily closing price data of major world stock indices S\&P 500~($8273$ days), S\&P/TSX~($8255$ days), EURO STOXX 50 ($8455$ days), FTSE 100~($8298$ days), DAX~($8302$ days), CAC 40~($8331$ days), TOPIX~($8065$ days) and Hang Seng ($8108$ days) are obtained from Bloomberg terminal
\footnote{Bloomberg terminal is a computer software provide by Bloomberg L.P.~(\url{https://about.bloomberg.co.jp/}). We retrieve the datasets in November, 2021.}. 
Ticker codes for these indices are SPX, SPTSX, SX5E, UKX, DAX, CAC, TPX, and HSI, respectively in this order, and the data period of all indices was from Jan. 1989 to Oct. 2021. 

For each index, the dataset is divided by month, with the first month being Jan.~1989 and the 394th month being Oct.~2021: in each month indexed by $i$, the daily closing prices of length $d_i$ are concatenated to form an observed covariate vector $X_i=(x_{i1},x_{i2},\ldots,x_{id_i}) \in \mathbb{R}^{d_i}$. For each index, 
the mean value of $d_i$ with its standard deviation is
$21.0 \pm 1.2$ for S\&P 500, 
$21.0 \pm 1.0$ for S\&P/TSX, 
$21.5 \pm 1.1$ for EURO STOXX 50, 
$21.1 \pm 1.2$ for FTSE 100, 
$21.1 \pm 1.3$ for DAX, 
$21.1 \pm 1.2$ for CAC 40, 
$20.5 \pm 1.2$ for TOPIX and 
$20.6 \pm 1.4$ for Hang Seng.  
For each month, we let the label $Y$ be 1 if the next month-end closing price increases from the current month and 0 otherwise.

\item \textbf{Prediction}:
We consider predicting the next month-end prices of $n_{\text{test}}=202$ months from Jan.~2005 to Oct.~2021 as the test set, by leveraging the most recent $n_{\text{train}}=192$ months from each test month as the training set. 
For instance, let the query $X_*$ be the daily closing prices in Jan.~2010: we predict whether the month-end closing price of Feb.~2010 rises or falls by leveraging the most recent $192$ months, i.e., those from Jan.~1994 to Dec.~2009. 
Note that the prediction always employs the most recent $192$ months but not the fixed $192$ months. 
The above query, representing the prices in Jan. 2010, is also used as a training sample for predicting the prices of the subsequent months. 

\item \textbf{Evaluation by accuracy}: All classifiers are evaluated by accuracy, i.e., the concordance rate between predicted labels and test labels.

\item \textbf{Evaluation by cumulative return for virtual trading}: 
Let $t \in \{ 1,2,\ldots,n_\text{test}\}$ be the month index for the test set; $t=1$ corresponds to Jan.~2005.
For month $t$, buy the stock index if the predicted label for the month $t$ (indicating whether the month-end price of the month $t+1$ would rise or fall) is $1$, and sell it otherwise. The return for $t$-th month from the virtual trading is
$$
    (\text{Buy}) \quad R_t=1+\frac{e_{t+1}-e_t}{e_t}, 
    \hspace{3em}
    (\text{Sell}) \quad R_t=1-\frac{e_{t+1}-e_t}{e_t},
$$
where $e_t$ denotes the stock index price at month-end of month $t$. 
The cumulative return $\prod_{\tau=1}^t R_{\tau}$ up to month $t=1,2,\ldots,n_{\text{test}}$ is used as a score~\citep{shen2015portfolio}. 
A higher return is better.

\item \textbf{Parametric regression function}: MS-$k$-NN (poly.) employs the polynomial function of degree $2$: $\theta_0+\theta_1 r+\theta_2 r^2$. 
MS-$k$-NN (logi.) and LRLR employ polynomial function of degree $2$ with the sigmoid function $\sigma(z)=(1+\exp(-z))^{-1}$: $\sigma(\theta_0'+\theta_1'r+\theta_2'r^2)$. 
In particular, 
MS-$k$-NN~(logi.) optimizes the regression function by minimizing a variant of the squared loss function $\sum_{j=1}^{J}\left\{\sigma^{-1}\big(\hat{\eta}_{k_j}^{(k\text{NN})} (X_*)\big) - \sigma^{-1}\big(v_{\text{logi.}}(r_{k_j})\big)\right\}^2$, using the logit function $\sigma^{-1}(z)=\log\frac{z}{1-z}$.

\item \textbf{Hyper-parameter tuning}: In order to determine hyper-parameters in $k$-NN and MS-$k$-NN (poly., logi.), 
we perform a walk forward testing~\citep{katz2000encyclopedia}.
Specifically, for month $t$ in the test set period, the parameter that yields the maximum accuracy on months $t-24$ to $t-1$ is chosen, while months $t-192, t-191, \ldots, t-25$ are used for the actual training.
We select parameter $k$ in $k$-NN from $\{1,2,\cdots,30\}$. As to MS-$k$-NN (poly., logi.), instead of choosing $\boldsymbol{k}=(k_1,k_2,\cdots,k_J)$, let $k_1=5$, $J=5$, and only one parameter $k_{\max}$ be selected from $\{20,30,50,80,120\}$. $k_2,k_3,\cdots,k_J$ are specified by $k_j:=k_1+\left\lfloor(j-1)(k_{\max}-k_1)/(J-1)\right\rfloor$ such that $\{k_j\}_{j=1}^J$ is an arithmetic sequence.
\end{itemize}

\subsection{Results}
The accuracy and the cumulative return for virtual trading are shown in Table~\ref{table:real_accuracy} and Figure~\ref{fig:real_cumulative_return}, respectively. Note that for random prediction, the accuracy and cumulative return are both the average of corresponding results obtained by 30 experiments.

As to accuracy, both LRLR with $w(r)=1$ and $w(r)=1/r$ outperform random prediction, $k$-NN, MS-$k$-NN (poly.) and MS-$k$-NN (logi.) on S\&P 500, S\&P/TSX, FTSE 100, DAX, CAC 40 and Hang Seng Index, while both of them demonstrate lower performance than some baselines on EURO STOXX 50 and TOPIX. On all stock indices, LRLR with $w(r)=1/r$ achieves higher or the same performance than $w(r)=1$. This result shows that $w(r)=1/r$ is superior to $w(r)=1$ for LRLR in terms of accuracy.

According to the cumulative return displayed in Figure~\ref{fig:real_cumulative_return}, LRLR with $w(r)=1/r$ always performs better or the same as LRLR with $w(r)=1$, while it also outperforms random prediction, $k$-NN, MS-$k$-NN (poly.) and MS-$k$-NN (logi.) on S\&P 500, S\&P/TSX, FTSE 100, DAX, CAC 40 and Hang Seng Indices. These are consistent with the results obtained from the comparison of accuracy.

While LRLR outperforms existing methods for most indices, their scores are degraded for EURO STOXX~(EU) and TOPIX~(Japan): 
EURO STOXX is distinct from other indices as it combines indices for a number of countries in the EU, and the Japanese market is often considered to be different from the international market. 
See \citet{fama2012size}.

\begin{table}[!ht]
\centering
\caption{Predictive classification accuracy. A higher score is better: the best and the second-best are bolded and underlined, respectively.}
\label{table:real_accuracy}
\scalebox{0.95}{
\begin{tabular}{lcccccccc}
\toprule
                    & S\&P 500 & S\&P/TSX & EURO S. 50 & FTSE 100 & DAX & CAC 40 & TOPIX & Hang Seng  \\
\midrule
random              & $0.492$ & $0.495$ & $0.498$   & $0.482$ & $0.492$ & $0.490$ & $0.493$         & $0.486$ \\
$k$-NN              & $0.574$ & $0.594$ & $0.510$  & $0.500$ & $0.530$ & $0.525$ & $0.500$ & $0.564$ \\
MS-$k$-NN (poly.)           & $0.604$ & $0.559$ & $\first{0.525}$           & $0.485$ & $0.545$ & $0.495$ & $\second{0.515}$ & $0.530$ \\
MS-$k$-NN (logi.)  & $0.604$ & $0.574$ & $0.510$ & $0.490$ & $0.525$ & $0.485$ & $\first{0.535}$ & $0.530$ \\
\rowcolor{gray!10}
LRLR ($w(r)=1$)     & $\first{0.649}$ & $\second{0.609}$ & $0.505$  & $\second{0.574}$  & $\first{0.609}$ & $\second{0.550}$ & $0.465$ & $\first{0.574}$ \\
\rowcolor{gray!10}
LRLR ($w(r)=1/r$)   & $\first{0.649}$ & $\first{0.619}$ & $\second{0.515}$   & $\first{0.584}$   & $\first{0.609}$ & $\first{0.554}$  & $0.475$ & $\first{0.574}$ \\
\bottomrule
\end{tabular}
}
\end{table}

\begin{figure}[p]
\centering 
\begin{minipage}{0.48\textwidth}
\centering
\includegraphics[width=0.85\textwidth]{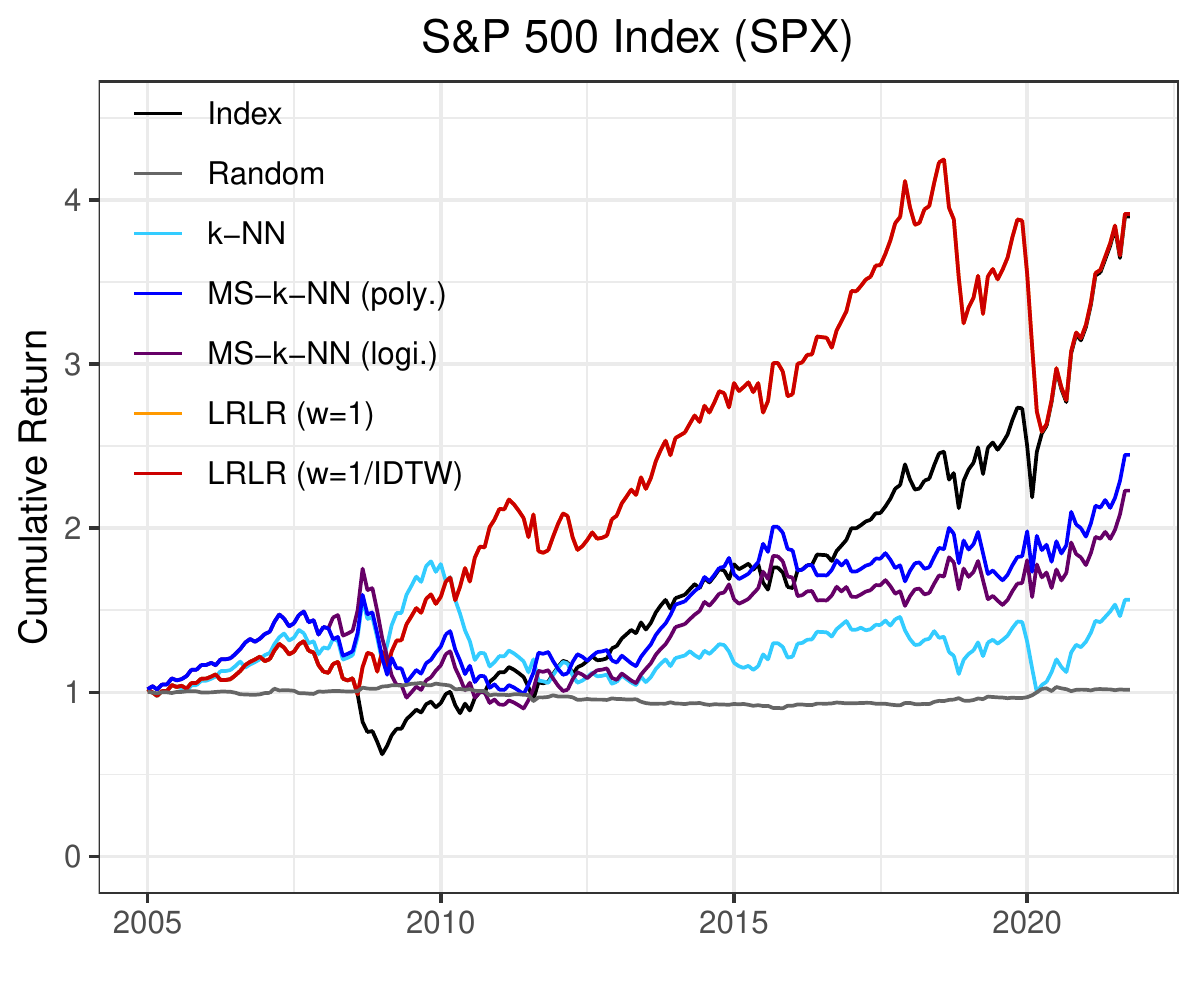}
\end{minipage}
\begin{minipage}{0.48\textwidth}
\centering 
\includegraphics[width=0.85\textwidth]{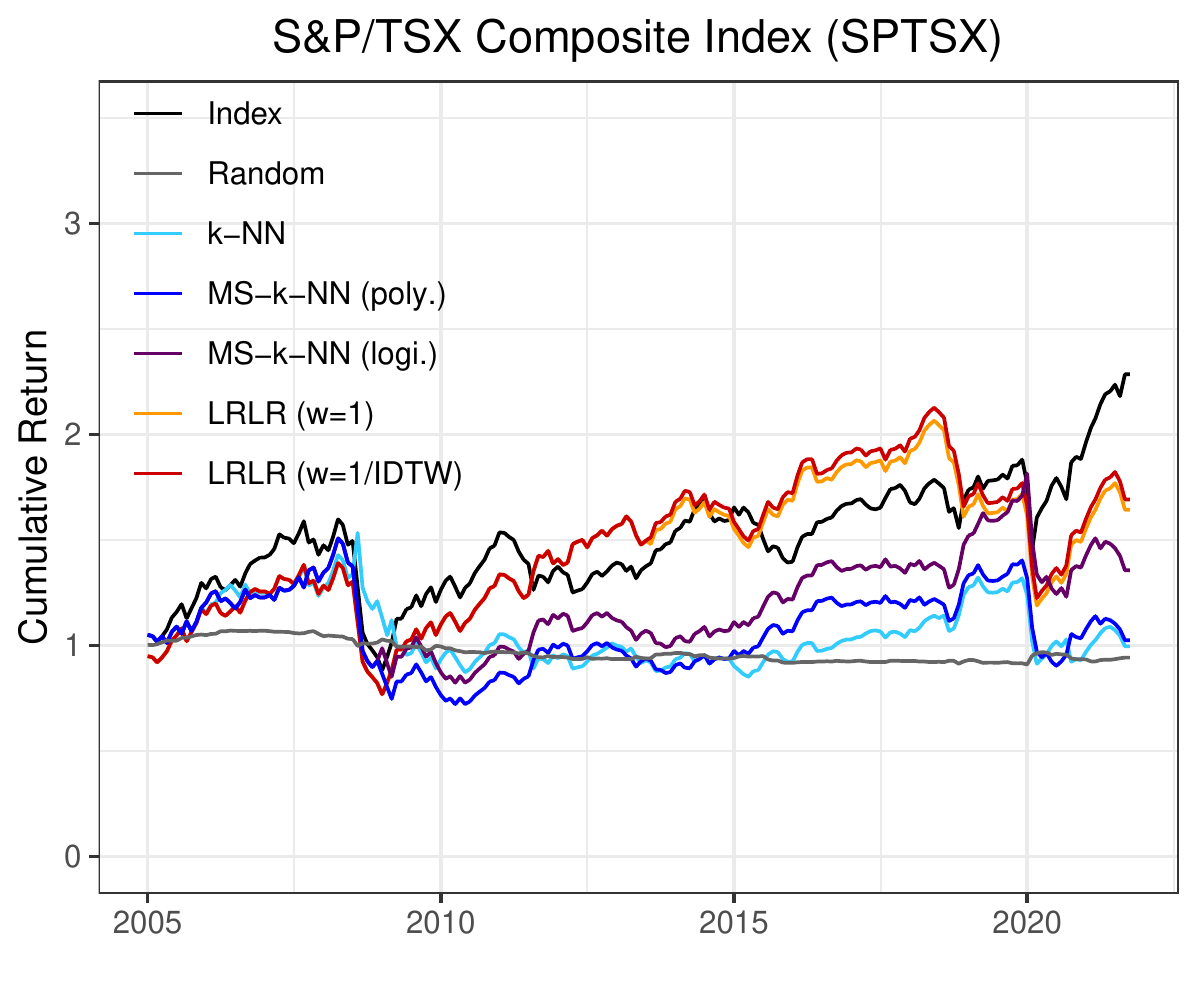}
\end{minipage} \\
\begin{minipage}{0.48\textwidth}
\centering
\includegraphics[width=0.85\textwidth]{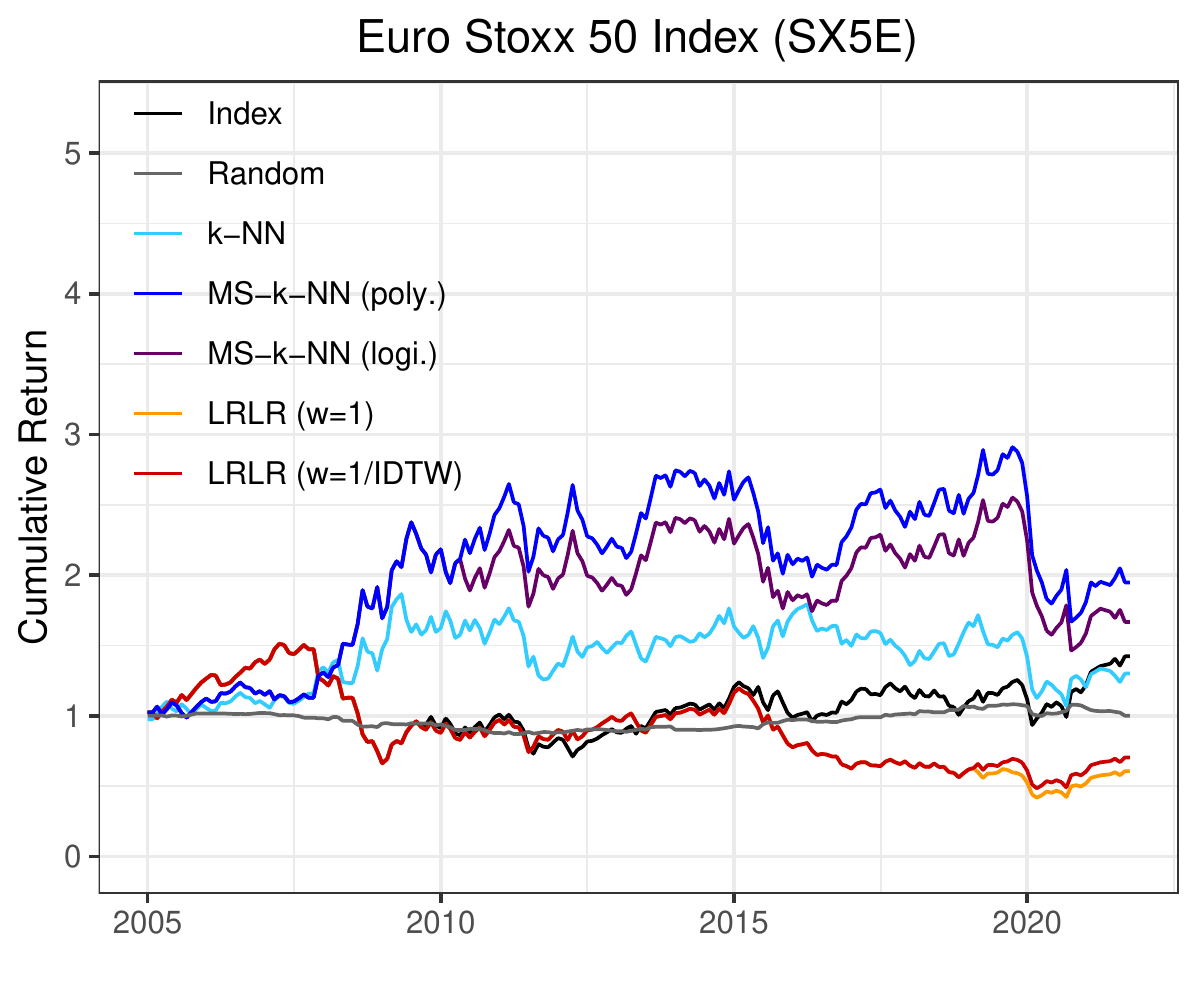}
\end{minipage}
\begin{minipage}{0.48\textwidth}
\centering 
\includegraphics[width=0.85\textwidth]{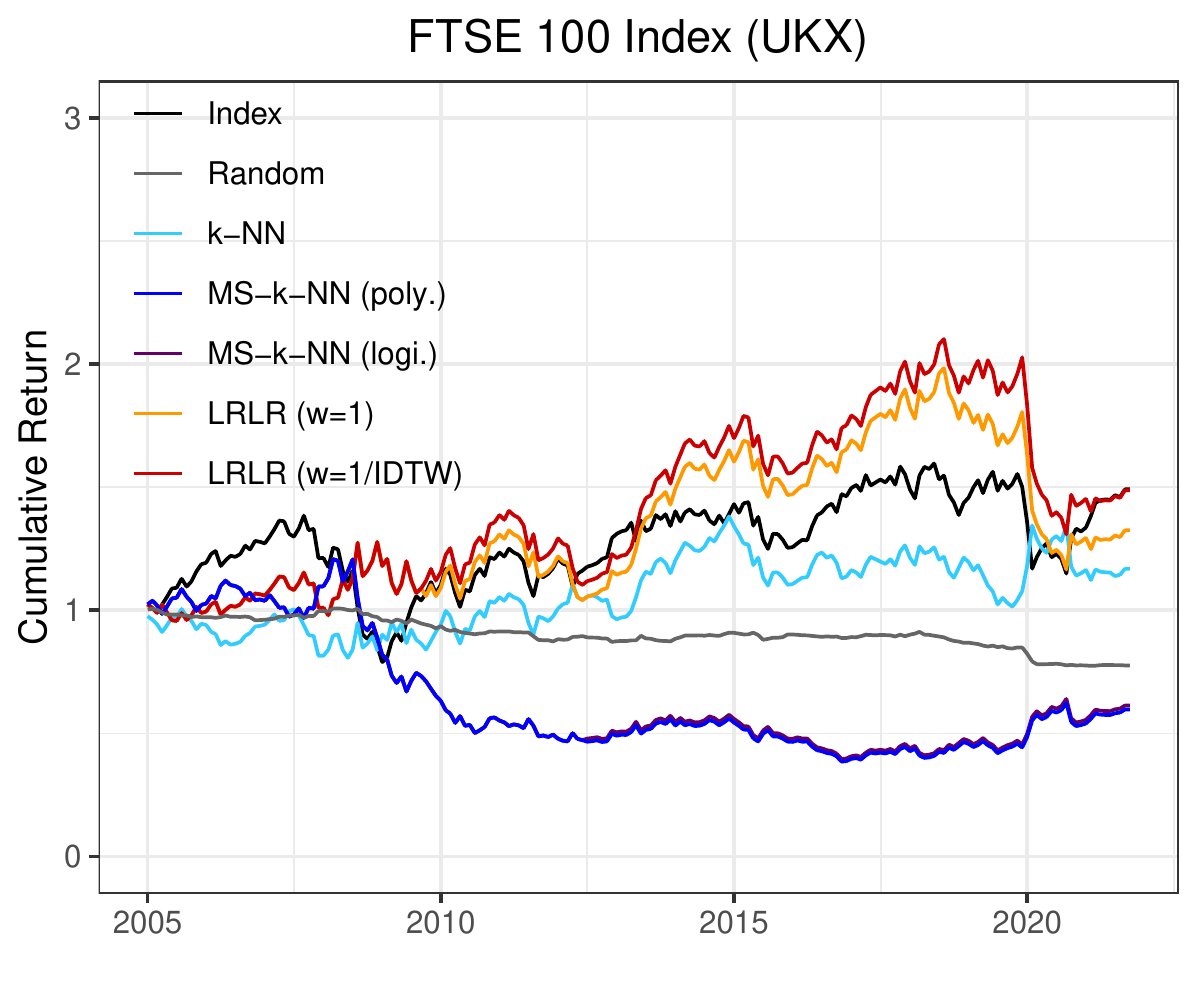}
\end{minipage} \\
\begin{minipage}{0.48\textwidth}
\centering
\includegraphics[width=0.85\textwidth]{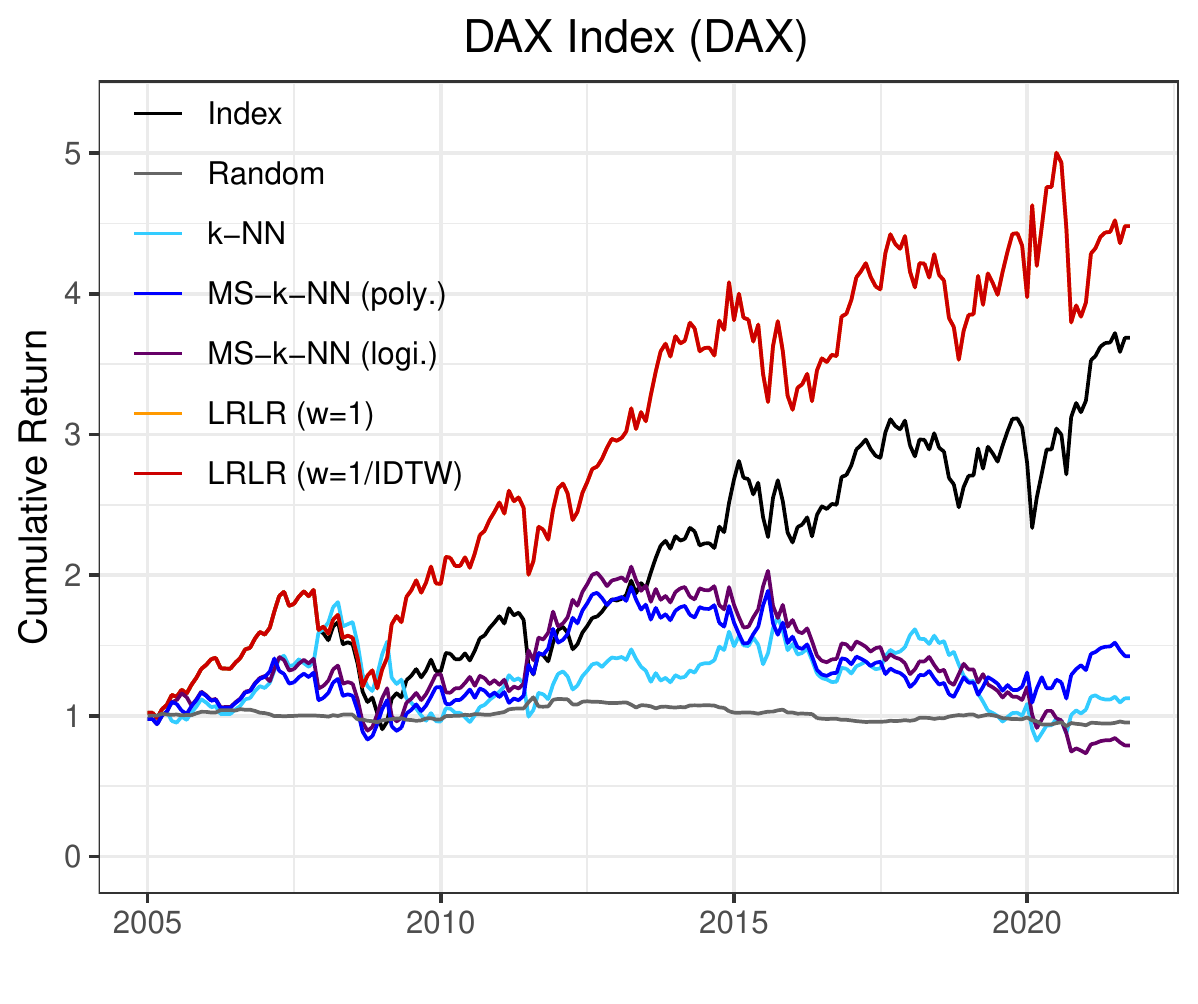}
\end{minipage}
\begin{minipage}{0.48\textwidth}
\centering 
\includegraphics[width=0.85\textwidth]{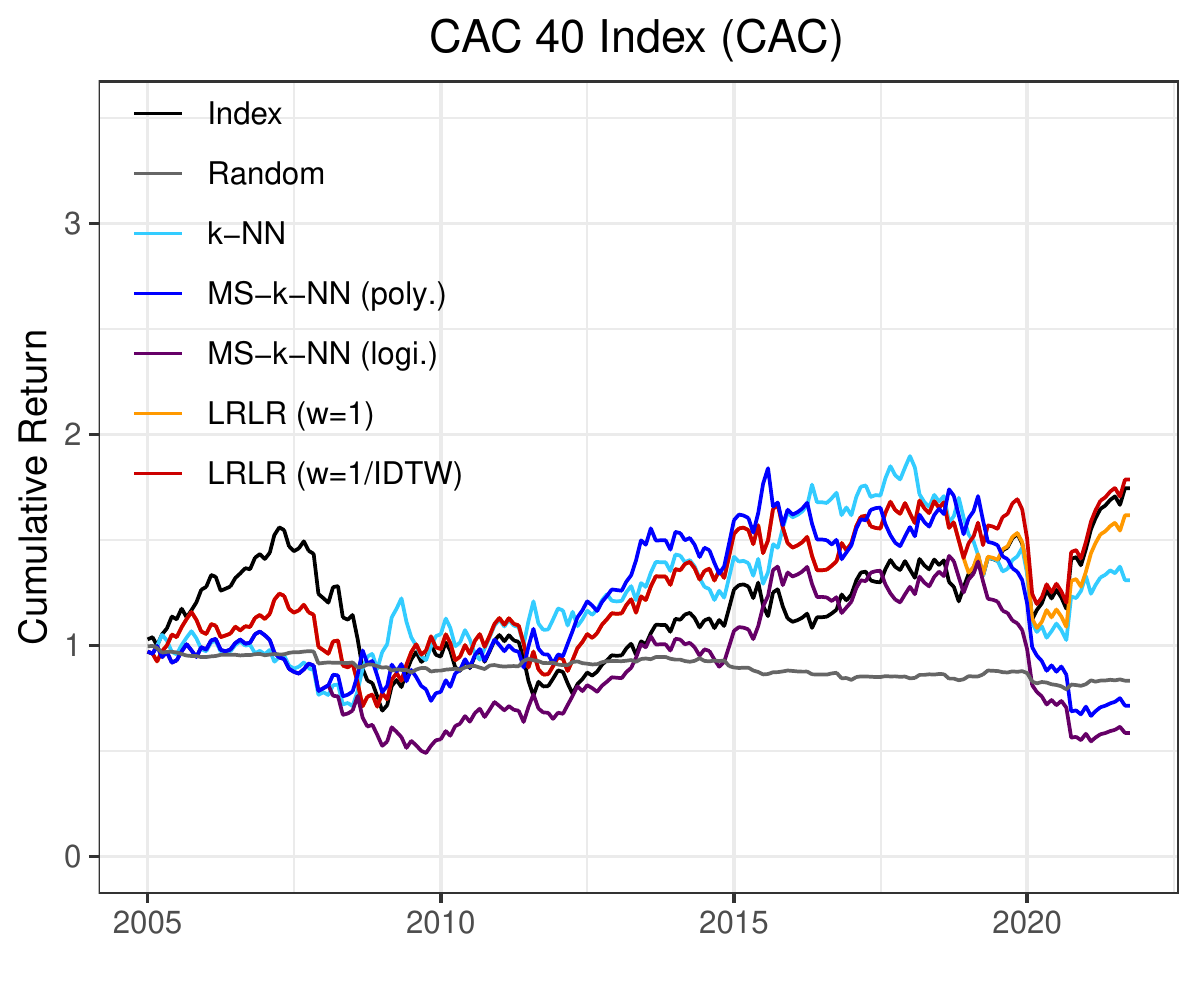}
\end{minipage} \\
\begin{minipage}{0.48\textwidth}
\centering
\includegraphics[width=0.85\textwidth]{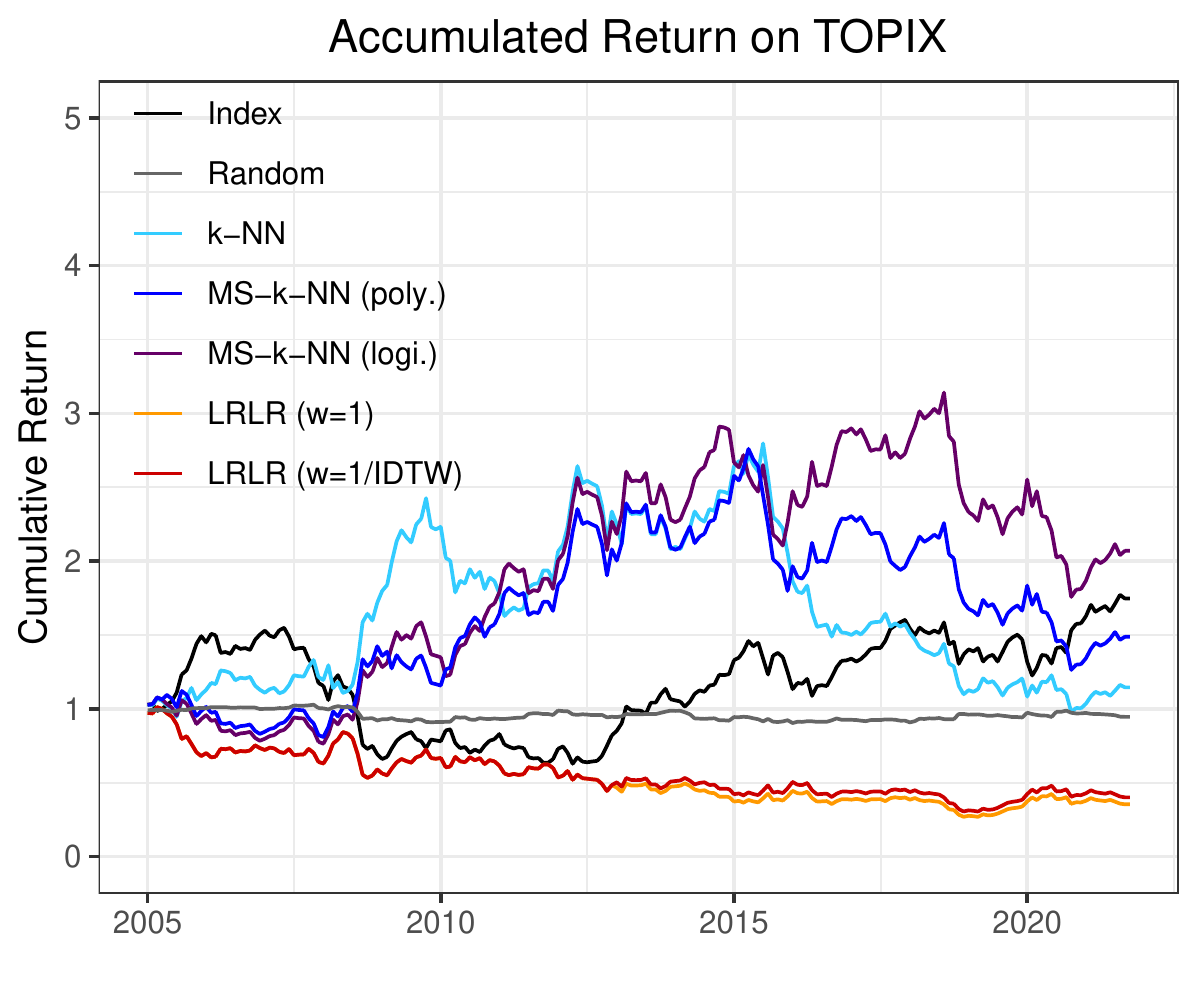}
\end{minipage}
\begin{minipage}{0.48\textwidth}
\centering 
\includegraphics[width=0.85\textwidth]{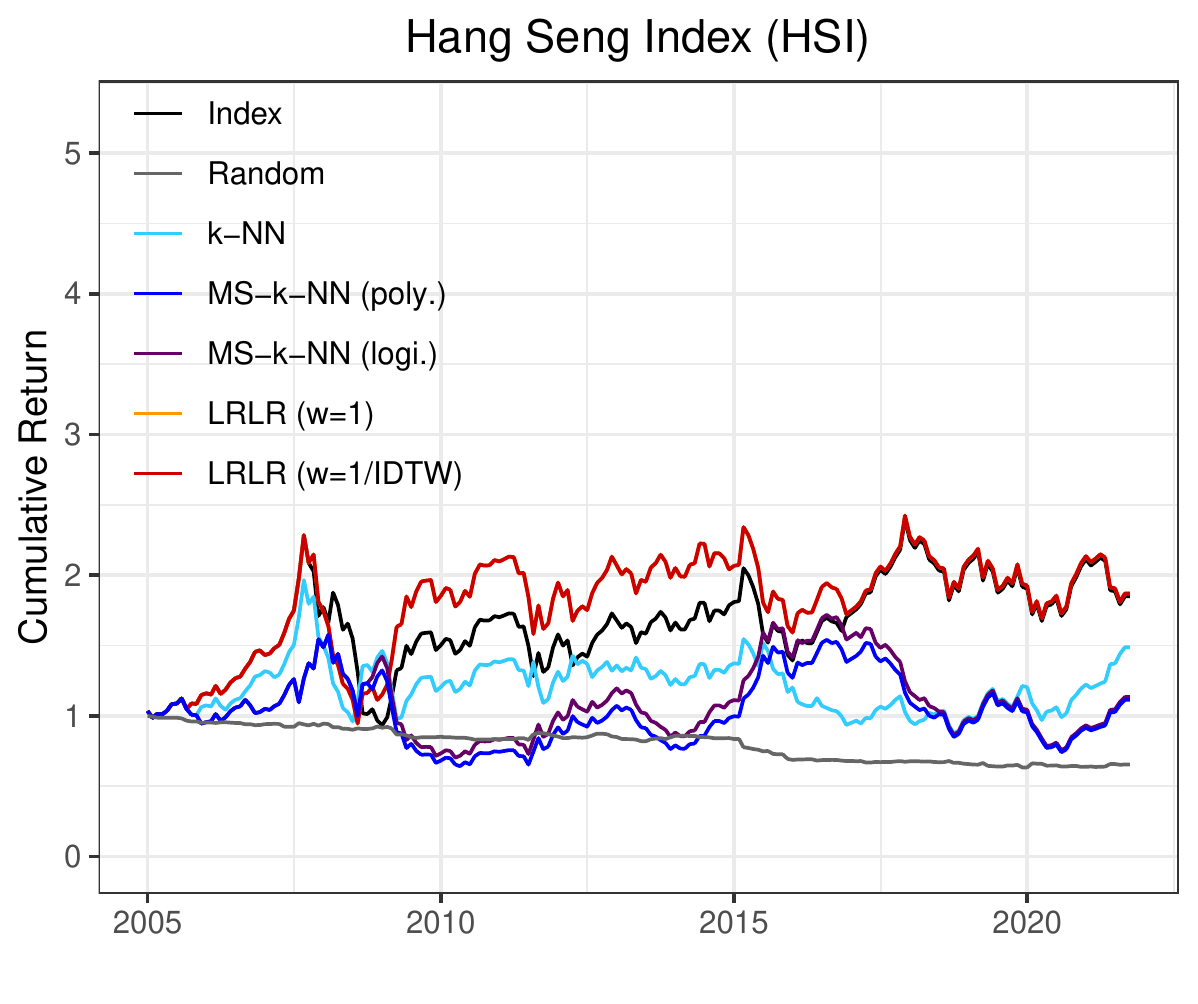}
\end{minipage} 
\caption{Cumulative return for virtual trading. A higher return is better. The black line (Index) shows the cumulative return when simply holding the index from the beginning of the test period: Jan. 2005.}
\label{fig:real_cumulative_return}
\end{figure}

\section{Conclusion}
\label{sec:conclusion}

We proposed a local radial regression~(LRR) and its logistic regression variant called a local radial logistic regression~(LRLR). 
LRR/LRLR combines the existing local polynomial regression~(LPoR) and multiscale $k$-nearest neighbor~(MS-$k$-NN). 
We proved the convergence rate of the $L^2$ risk for LRR. 
In our numerical experiments, LRLR outperforms these existing bias-correction approaches LPoR and MS-$k$-NN. 
We applied LRLR to real-world datasets of eight major world stock indices, and classified whether month-end closing prices rise or fall: predictive classification accuracy is improved, compared to existing methods.

\section*{Acknowledgement}
AO is supported by JSPS KAKENHI (21K17718). 
HS is supported by JSPS KAKENHI (20H04148). 
HS and AO are supported by JST CREST (JPMJCR21N3). 
We would like to thank Takuma Tanaka for helpful discussions.

\appendix

\section[Bias Correction via LPoR and MS-k-NN]{Bias Correction via LPoR and MS-$k$-NN}
\label{app:asymptotic_biases}

This section briefly summarizes the asymptotic higher-order bias of the kernel smoother~(KS) and $k$-NN, and their bias correction via LPoR and MS-$k$-NN. 
See \citet{tsybakov2009introduction}, \citet{samworth2012optimal} and \citet{okuno2020extrapolation} for more rigorous descriptions.

Here, assume that the function $\eta(X)=\mathbb{P}(Y=1 \mid X)$ is $q$-times continuously differentiable for some $q \in \mathbb{N}$, i.e., $\eta(X)$ has a Taylor expansion
\[
    \eta(X) = \eta(X_*) + b_*^{(q)}(X-X_*) + o(\|X-X_*\|_2^q)
\]
for some polynomial $b_*^{(q)}:\mathbb{R}^d \to \mathbb{R}$ of degree $q$ satisfying $b_*^{(q)}(\bs 0)=0$. 
Also assume $h=r_k \: (\approx 0)$ so that the kernel smoother~(\ref{eq:kernel_smoother}) and the $k$-NN estimator~(\ref{eq:kNN}) are compatible, and 
\[
q \ge 2, \quad \text{i.e., $\eta$ is highly smooth}. 
\]

Then, the law of large numbers proves
\[
    \hat{\eta}^{(\text{KS})}_h(X_*)
    ,\,
    \hat{\eta}^{(k\text{NN})}_k(X_*)
\quad \approx \quad 
    \mathbb{E}(Y \mid \|X-X_*\|_2 \le h) 
\quad \approx \quad 
    \eta(X_*) + 
    \underbrace{
    \mathbb{E}(b_*^{(q)}(Z) \mid \|Z\|_2 \le h)
    }_{=: \, (\star) \, = \, O(h^{2})}
    +
    o(h^q):
\]
see, e.g., eq.~(31) in \citet{okuno2020extrapolation} Supplement~F.2 with $r=h$, for the evaluation of the bias term $(\star)=O(h^2)$. 
Both LPoR and MS-$k$-NN eradicate the bias term $(\star)$ in the above formula:
$\hat{u}(Z)$ approximates $\eta(X_*) + b_*^{(q)}(Z)$ and
$\hat{v}(r)$ approximates $\eta(X_*)+\mathbb{E}(b_*^{(q)}(Z) \mid \|Z\|_2 \le r)$. 
Namely, the regression functions in LPoR and MS-$k$-NN are trained to predict not only the target ground-truth label probability $\eta(X_*)$ but also the bias term $(\star)$: substituting $Z=\bs 0$ into $\hat{u}(Z)$ and $r=0$ into $\hat{v}(r)$ eradicates the bias as
\[
    \hat{\eta}^{(\text{LPoR})}_h (X_*)
    ,\,
    \hat{\eta}^{(\text{MS}k\text{NN})}_{k_1,k_2,\ldots,k_J}(X_*)
\quad \approx \quad 
    \eta(X_*) 
    +
    o(h^q).
\]
As $(\star)=O(h^2)$ is the leading term, LPoR and MS-$k$-NN attain the faster convergence rate $o(h^q)$ than the rate $O(h^2)$ of kernel smoother and $k$-NN.

\section{Proofs}
\label{app:proof}

Proof of Theorem~\ref{theo:convergence_rate_LRLR} is provided in Appendix~\ref{proof:theo:convergence_rate_LRLR} with supporting propositions shown in Appendix~\ref{app:supporting_propositions}. 
Proof of Example~\ref{ex:C-3} is shown in Appendix~\ref{proof:ex:C-3}.

\subsection{Proof of Theorem~\ref{theo:convergence_rate_LRLR}}
\label{proof:theo:convergence_rate_LRLR}

In this section, we prove Theorem~\ref{theo:convergence_rate_LRLR} with supporting propositions shown in Appendix~\ref{app:supporting_propositions}. 

For the proof, we consider a decomposition of the $L^2$ risk:
\begin{align}
    \mathbb{E}_{\setDn}(\{\eta(X_*)-\hat{\eta}_n(X_*)\}^2)
    &=
    \mathbb{E}_{\setDn}(\{
        (\eta(X_*)-\tilde{\eta}_n(X_*)) \, + \, (\tilde{\eta}_n(X_*) - \hat{\eta}_n(X_*)) \}^2) \nonumber \\
    & \le 
    4\max\{ \mathbb{E}_{\setDn}(\{\eta(X_*)-\tilde{\eta}_n(X_*)\}^2) \, , \, \mathbb{E}_{\setDn}(\{\tilde{\eta}_n(X_*) - \hat{\eta}_n(X_*)\}^2)\}, 
    \label{eq:risk_decomposition}
\end{align}
with 
\begin{align*}
    \tilde{\eta}(r)&:=\mathbb{P}(Y=1 \mid \|X-X_*\|_2=r), \quad 
    r_i:=\|X_{(i)}-X_*\|_2 \: (i=1,2,\ldots,n), \\
    \tilde{\eta}_n(X_*)
    &:=
    \begin{cases} 
    \tilde{f}_n(0),
    \quad 
    \tilde{f}_n := \argmin_{f \in \setF} \sum_{i=1}^{n} w_n(r_i) \{ \tilde{\eta}(r_i) - f(r_i)\}^2 & (\text{if }\event \text{ is satisfied}) \\ 0 & (\text{otherwise}) 
    \end{cases}, \quad \text{and} \\
    \hat{\eta}_n(X_*)
    &:=
    \begin{cases} 
    \hat{f}_n(0), 
    \quad 
    \hat{f}_n := \argmin_{f \in \setF} \sum_{i=1}^{n} w_n(r_i) \{ Y_{(i)} - f(r_i) \}^2 & (\text{if }\event \text{ is satisfied}) \\ 0 & (\text{otherwise}) 
    \end{cases}. 
\end{align*}

Herein, we evaluate the terms in the decomposition (\ref{eq:risk_decomposition}), by the following three steps.

\paragraph{Step 1: Derivation of the probability density $\tilde{\eta}(r)$.} 
To evaluate the terms in (\ref{eq:risk_decomposition}), we first derive the conditional probability density $\tilde{\eta}(r):=\mathbb{P}(Y=1 \mid \|X-X_*\|_2=r)$. 
However, this conditional probability is not simply obtained by Bayes' theorem as both denominator and numerator of the fraction $\mathbb{P}(Y=1,\|X-X_*\|_2=r)/\mathbb{P}(\|X-X_*\|_2=r)$ are  $0$ (i.e., $0/0$ cannot be defined). Thus we employ another derivation
\begin{align}
    \tilde{\eta}(r)
    :=
    \mathbb{P}(Y=1 \mid \|X-X_*\|_2=r)
    =
    \lim_{\delta \searrow 0}
    \underbrace{
        \frac{\mathbb{P}(Y=1,\|X-X_*\|_2 \in [r-\delta,r+\delta])}{\mathbb{P}(\|X-X_*\|_2 \in [r-\delta,r+\delta])}
    }_{=: \, \tilde{\eta}_{\delta}(r)};
    \label{eq:eta_evans}
\end{align}
see, e.g., Theorem~3 in \citet{evans1992measure} p.38 proving that (\ref{eq:eta_evans}) is compatible with the conditional probability. Using this derivation, we obtain an expression
\begin{align}
    \tilde{\eta}(r)
    &=
    \eta(X_*) 
    +
    \sum_{c=1}^{\omega} b_{\mu,\eta\mu,c}^* r^{2c}
    +
    \tilde{C}_{\mu,\eta \mu}^*(r) r^{\beta}
    \quad (\forall r \in [0,R/2]), \label{eq:taylor_conditional_probatility} \\
    \omega &:= \lfloor \beta/2 \rfloor, \nonumber 
\end{align}
for some constants $\{b_{\mu,\eta\mu,c}^*\}_{c=1}^{\omega} \subset \mathbb{R}$ and $\sup_{r \le R/2}|\tilde{C}_{\mu,\eta\mu}^*(r)|<\infty$. See Proposition~\ref{prop:explicit_form_of_eta_r} in Appendix~\ref{app:supporting_propositions}.

\paragraph{Step 2: Explicit forms of $\tilde{\eta}_n(X_*),\hat{\eta}_n(X_*)$.} 
As the estimators $\tilde{\eta}_n(X_*),\hat{\eta}_n(X_*)$ are intercepts of polynomial regressions, their explicit forms can be obtained by 
a simple matrix algebra. 
Using a vector $1_{N_n}:=(1,1,\ldots,1) \in \mathbb{R}^{N_n}$, 
an identity matrix $I_{N_n}=\text{diag}(1_{N_n}) \in \{0,1\}^{N_n \times N_n}$, a projection matrix 
$\proj_{\tilde{R}_n}=\tilde{R}_n(\tilde{R}_n^{\top}\tilde{R}_n)^{-1}\tilde{R}_n^{\top}$ for a matrix 
\begin{align}
    \tilde{R}_n &:= \left(\begin{array}{cccc}
    r_{1}^2 & r_{1}^4 & \cdots & r_{1}^{2\omega} \\
    r_{2}^2 & r_{2}^4 & \cdots & r_{2}^{2\omega} \\
    \vdots & \vdots & \ldots & \vdots \\
    r_{N_n}^2 & r_{N_n}^4 & \cdots & r_{N_n}^{2\omega} \\
    \end{array} \right) \in \mathbb{R}^{N_n \times \omega}, \quad
    \tilde{e}_n:=\left(\begin{array}{c}
         \tilde{\eta}(r_1) \\
         \tilde{\eta}(r_2) \\
         \vdots \\
         \tilde{\eta}(r_{N_n}) \\
    \end{array}\right) \in [0,1]^{N_n}, \quad 
    \tilde{y}_n:=\left(\begin{array}{c}
        Y_{(1)} \\ Y_{(2)} \\ \vdots \\ Y_{(N_n)} \\ 
    \end{array}\right)  \in \mathbb{R}^{N_n}, \nonumber \\
    \text{and } \: 
    &\tilde{\rho}_n
    =
    (\rho_{n,1},\rho_{n,2},\ldots,\rho_{n,N_n})^{\top}
    :=
    \frac{(I_{N_n}-\proj_{\tilde{R}_n}) 1_{N_n}}{\langle 1_{N_n}, (I_{N_n}-\proj_{\tilde{R}_n}) 1_{N_n}\rangle} 
    =
    \frac{(I_{N_n}-\proj_{\tilde{R}_n}) 1_{N_n}}{N_n-\zeta_n} \in \mathbb{R}^{N_n},
    \label{eq:zeta_n}
\end{align}
we have
\begin{align}
\tilde{\eta}_n(X_*)
    &=
    \langle \tilde{\rho}_n, \tilde{e}_n \rangle
    \quad\text{and}\quad
\hat{\eta}_n(X_*)
    =
    \langle \tilde{\rho}_n, \tilde{y}_n \rangle
    \qquad 
    (\text{if $\event$}),
    \label{eq:expression_eta}
\end{align}
by following the same calculation as Appendix~E in \citet{okuno2020extrapolation}. 
Here $\{\rho_{n,i}\}_{i=1}^{N_n}$ are interpreted as weights, and
$\tilde{\eta}_n(X_*)$ and $\hat{\eta}_n(X_*)$ are expressed as the weighted averages of $\tilde \eta(r_i)$'s and $Y_{(i)}$'s, respectively. 
The weights satisfy $\sum_{i=1}^{N_n}\rho_{n,i}=1$ but are not restricted to non-negative, so there is a concern if some of the weights may diverge. Thus we need the condition (C-3) in Section~\ref{subsec:improved_rate}.
Note that $V,\bs \varphi_{n,\bs k},\bs R,\bs z$ in \citet{okuno2020extrapolation} correspond to the above defined $N_n,\tilde{y}_n \: (\text{and }\tilde{e}_n),\tilde{R}_n,\tilde{\rho}_n$, respectively. 

For the remaining case, 
\begin{align}
\tilde{\eta}_n(X_*)=0 \quad \text{and} \quad 
\hat{\eta}_n(X_*)=0 
\qquad (\text{if $\lnot \event$}).
\label{eq:expression_eta_exception}
\end{align}

\paragraph{Step 3: Evaluation of the terms in the decomposition (\ref{eq:risk_decomposition}).} 

Using (\ref{eq:taylor_conditional_probatility})--(\ref{eq:expression_eta_exception}), we evaluate the two terms in the decomposition~(\ref{eq:risk_decomposition}). 

\begin{itemize}
    \item \textbf{Evaluation of the first term}. 
    Firstly, we consider the case that $\event$ is satisfied. 
    With vectors
    \begin{align*}
        \tilde{b}^*:=(b_{\mu,\eta\mu,1}^*,b_{\mu,\eta\mu,2}^*,\cdots,b_{\mu,\eta\mu,\omega}^*)^{\top} \in \mathbb{R}^{\omega},
        \quad 
        \tilde{c}_n^*:=(\tilde{C}_{\mu,\eta\mu}^*(r_1),\tilde{C}_{\mu,\eta\mu}^*(r_2),\cdots,\tilde{C}_{\mu,\eta\mu}^*(r_{N_n}))^{\top} \in \mathbb{R}^{N_n},
    \end{align*}
    the expression (\ref{eq:taylor_conditional_probatility}) gives 
    \begin{align}
    \tilde{e}_n=\eta(X_*)1_{N_n}+\tilde{R}_n \tilde{b}^* + \tilde{c}_n^* \tilde{r}_n^{\beta}.
    \label{eq:en}
    \end{align}
    
Considering the equations 
\begin{align}
\langle \tilde{\rho}_n, \alpha 1_{N_n}\rangle=\alpha \quad (\forall \alpha \in \mathbb{R}) 
\quad \text{and} \quad 
\langle \tilde{\rho}_n,\tilde{R}_n b\rangle=0 \quad (\forall b \in \mathbb{R}^{N_n}),
\label{eq:rhon_equations}
\end{align}
where the latter equation is proved by $\langle (I_{N_n}-\proj_{\tilde{R}_n})1_{N_n}, \tilde{R}_n b \rangle=\langle \tilde{R}_n^{\top}(I_{N_n}-\proj_{\tilde{R}_n}) 1_{N_n}, b\rangle = \langle (\tilde{R}_n^{\top}-\tilde{R}_n^{\top})1_{N_n}, b\rangle = 0$, we have 
\begin{align}
    |\eta(X_*)
    -
    \tilde{\eta}_n(X_*)|
    &\overset{(\ref{eq:expression_eta}),(\ref{eq:en}),(\ref{eq:rhon_equations})}{=}
    |\langle \tilde{\rho}_n, \eta(X_*) 1_{N_n} \rangle
    -
    \langle \tilde{\rho}_n,\tilde{e}_n \rangle| \nonumber \\
    &=
    |\langle \tilde{\rho}_n, \eta(X_*)1_{N_n}-\tilde{e}_n \rangle| \nonumber \\
    &\overset{(\ref{eq:en})}{=} 
    |\langle \tilde{\rho}_n, \tilde{R}_n\tilde{b}^* + \tilde{c}_n^* \tilde{r}_n^{\beta}\rangle| \nonumber \\
    &=
    |\langle \tilde{\rho}_n, \tilde{R}_n \tilde{b}^*\rangle + \langle \tilde{\rho}_n, \tilde{c}_n^* \rangle \tilde{r}_n^{\beta}| \nonumber \\
    &\overset{(\ref{eq:rhon_equations})}{=}
    |\langle \tilde{\rho}_n, \tilde{c}_n^* \rangle| \tilde{r}_n^{\beta} \nonumber \\
    &\le 
    (\|\tilde{\rho}_n\|_2^2 \|\tilde{c}_n^*\|_2^2)^{1/2} \tilde{r}_n^{\beta} \quad (\because \, \text{Cauchy-Schwarz}) \nonumber \\
    &\overset{(\star)}{\le} 
    \left( \frac{1}{(1-\varphi)N_n} \cdot N_n \|\tilde{c}_n^*\|_{\infty}^2 \right)^{1/2} \tilde{r}_n^{\beta} \nonumber \\
    &\le 
    (1-\varphi)^{-1/2} \|\tilde{c}_n^*\|_{\infty} \tilde{r}_n^{\beta} \nonumber \\
    &\le 
    (1-\varphi)^{-1/2} \left\{ 
    \sup_{r \le R/2} |\tilde{C}_{\mu,\eta\mu}^*(r)|
    \right\} \tilde{r}_n^{\beta}. \nonumber
\end{align}    
The inequality ($\star$) is proved by 
\begin{align}
    \|\tilde{\rho}_n\|_2^2
    &=
    \sum_{i=1}^{N_n} \rho_{n,i}^2
    =
    \bigg\| 
        \frac{(I_{N_n}-\proj_{\tilde{R}_n})1_{N_n}}{ \langle 1_{N_n},(I_{N_n}-\proj_{\tilde{R}_n})1_{N_n}\rangle  }
    \bigg\|_2^2
    =
    \frac{1}{\langle 1_{N_n},(I_n-\proj_{\tilde{R}_n}) 1_{N_n} \rangle} 
    =
    \frac{1}{N_n-\zeta_n} \nonumber \\
    &\le
    \frac{1}{N_n-\varphi N_n}
    =
    \frac{1}{(1-\varphi)N_n}
    \label{eq:rho_ss}
\end{align}
and $\|c\|_2^2 =\sum_{i}c_i^2 \le N_n \max_i |c_i|^2 \le N_n \|c\|_{\infty}^2$ for $c=(c_1,c_2,\ldots,c_{N_n})^{\top} \in \mathbb{R}^{N_n}$.

If $\event$ is not satisfied, $\eta(X_*)-\tilde{\eta}_n(X_*)=\eta(X_*) \in [0,1]$. 
Therefore, we have
\begin{align}
    \mathbb{E}_{\setDn}(\{\eta(X_*)-\tilde{\eta}_n(X_*)\}^2) 
    &\le 
    \mathbb{E}_{\setDn}(\{\eta(X_*)-\tilde{\eta}_n(X_*)\}^2 \mid \event) \underbrace{\mathbb{P}(\event)}_{\le 1} \nonumber \\
    &\hspace{3em}+
    \underbrace{\mathbb{E}_{\setDn}(\{\eta(X_*)-\tilde{\eta}_n(X_*)\}^2 \mid \lnot \event)}_{\le 1} \mathbb{P}(\lnot \event) \nonumber \\
    &\le 
    \mathbb{E}_{\setDn}(\{\eta(X_*)-\tilde{\eta}_n(X_*)\}^2 \mid \event) 
    +
    \mathbb{P}(\lnot \event) \nonumber \\
    &\le     
    (1-\varphi)^{-1/2} \left\{ 
    \sup_{r \le R/2} |\tilde{C}_{\mu,\eta\mu}^*(r)|
    \right\}^2 
    \tilde{r}_n^{2\beta}
    +
    \mathbb{P}(\lnot \event). \label{eq:evalution_former}
\end{align}

    \item \textbf{Evaluation of the second term}. 
    Here, we consider the case that $\event$ is satisfied. 
    For the quantity 
    \[
    \Delta_n:=\tilde{\eta}_n(X_*)-\hat{\eta}_n(X_*)=\langle \tilde{\rho}_n, \tilde{e}_n - \tilde{y}_n \rangle =\sum_{i=1}^{N_n} \rho_{n,i}\{\tilde{\eta}(r_i)-Y_{(i)}\}
    \]
    to be evaluated, we have the conditional expectation
    \begin{align}
        \mathbb{E}(\Delta_n \mid r_1,r_2,\ldots,r_n)
        =
        \sum_{i=1}^{n} \rho_{n,i} \underbrace{\{\tilde{\eta}(r_i) - \mathbb{E}(Y_{(i)} \mid r_i)\}}_{=0}=0
        \label{eq:E_Delta}
    \end{align}
    and the conditional variance
    \begin{align}
        \mathbb{V}(\Delta_n \mid r_1,r_2,\ldots,r_n) = \sum_{i=1}^{N_n} \rho_{n,i}^2 \mathbb{V}(Y_{(i)} \mid r_i) 
        \overset{(\star)}{\le} \frac{1}{4}\sum_{i=1}^{N_n} \rho_{n,i}^2
        \overset{(\ref{eq:rho_ss})}{\le} \frac{1}{4(1-\varphi) N_n}
        \label{eq:V_Delta}
    \end{align}
    where the inequality ($\star$) follows from the simple fact that the variance of any random variable taking value in $[0,c]$ is bounded by $c^2/4$ ($c\ge 0$). 
    Therefore, 
    \begin{align*}
        \mathbb{E}_{\setDn}(\{\tilde{\eta}_n(X_*)-\hat{\eta}_n(X_*)\}^2 \mid \event)
        &=
        \mathbb{E}_{\setDn}(\Delta_n^2 \mid \event) \\
        &=
        \mathbb{V}_{\setDn}(\Delta_n \mid \event)
        +
        \underbrace{\mathbb{E}_{\setDn}(\Delta_n \mid \event)^2}_{=\mathbb{E}_{\setDn}(\mathbb{E}(\Delta_n \mid r_1,r_2,\ldots, r_{N_n}) \mid \event)^2=0 \: (\because \eqref{eq:E_Delta})} \\
        &=
        \mathbb{V}_{\setDn}(\Delta_n \mid \event) \\
        &=
        \mathbb{E}(\underbrace{\mathbb{V}(\Delta_n \mid r_1,r_2,\ldots,r_n)}_{\le 1/\{4(1-\varphi)N_n\} \: (\because \eqref{eq:V_Delta}) } \mid \event) \\
        &\hspace{5em}+
        \mathbb{V}(\underbrace{\mathbb{E}(\Delta_n \mid r_1,r_2,\ldots,r_n)}_{=0 \: (\because \eqref{eq:E_Delta})} \mid \event) \\
        &\le 
        \frac{1}{4(1-\varphi)}
        \mathbb{E}\left( \frac{1}{N_n} \mid \event \right) \\
        &\le 
        \frac{1}{1-\varphi} \left\{\exp(-cn\tilde{r}_n^d) + \frac{1}{c'n\tilde{r}_n^d}\right\}
    \end{align*}
    for some $c,c'>0$, where the last inequality is obtained by Proposition~\ref{prop:prob_Nnomega} in Appendix~\ref{app:supporting_propositions}. 
    As $\tilde{\eta}_n(X_*)=\hat{\eta}_n(X_*)=0$ holds for the remaining case $\lnot \event$, we have
    \begin{align}
        \mathbb{E}_{\setDn}(\{\tilde{\eta}_n(X_*)-\hat{\eta}_n(X_*)\}^2)
        &=
        \mathbb{E}_{\setDn}(\{\tilde{\eta}_n(X_*)-\hat{\eta}_n(X_*)\}^2 \mid \event)\underbrace{\mathbb{P}(\event)}_{\le 1} \nonumber \\
        &\le 
        \frac{1}{1-\varphi}\left\{\exp(-cn\tilde{r}_n^d) + \frac{1}{c'n\tilde{r}_n^d}\right\}.
        \label{eq:evalution_latter}
    \end{align}

\end{itemize}

As the decay rate of $\mathbb{P}(\lnot \event)$ is exponential (see Proposition~\ref{prop:negation_probability} in Appendix~\ref{app:supporting_propositions}) and exponential terms are smaller than $\tilde{r}_n^{2\beta}$ and $1/n\tilde{r}_n^d$, substituting the above evaluations (\ref{eq:evalution_former}) and (\ref{eq:evalution_latter}) to (\ref{eq:risk_decomposition}) yields the assertion
\[
    \mathbb{E}_{\setDn}(\{\eta(X_*)-\hat{\eta}_n(X_*)\}^2)
    \le 
    C \: \max\left\{ 
     \, \tilde{r}_n^{2\beta}, \, \frac{1}{n \tilde{r}_n^d}
    \right\}
\]
for some $C>0$.

\subsection{Supporting Propositions}
\label{app:supporting_propositions}

Throughout this section, 
\[
\integralB(f;r):=\int_{B_r(X_*)} f(X) \diff X
\]
denotes an integral of the function $f:\mathbb{R}^d \to \mathbb{R}$ over a ball $B_r(X_*):=\{X \in \mathbb{R}^d \mid \|X-X_*\|_2 \le r\}$ with $r \ge 0$, 
$D^{\bs u}=\frac{\partial^{u_1+u_2+\cdots+u_d}}{\partial x_1^{u_1} \partial x_2^{u_2} \cdots \partial x_d^{u_d}}$ denotes a derivative operator equipped with a multiple index $\bs u=(u_1,u_2,\ldots,u_d) \in (\mathbb{N} \cup \{0\})^d$ satisfying $|\bs u|=u_1+u_2+\cdots+u_d, \bs u!=u_1!u_2!\cdots u_d!,u!=\prod_{j=1}^{u}j$, and 
$\Gamma(z):=\int_0^{\infty} t^{z-1}\exp(-t) \diff t$ denotes Gamma function. 
$\text{vol}(A)$ denotes the volume of a set $A \subset \mathbb{R}^d$, i.e., $\text{vol}(A):=\int_A \diff X$. 
$\binom{d}{j}=\frac{d!}{j!(d-j)!}$ denotes a binomial coefficient.

\bigskip 
With these symbols, the following Propositions hold.

\begin{prop}
\label{prop:integral}
    Let $0<\delta<r \le R/2$ with $\delta=o(r)$, 
    and let $f$ be a $\beta$-H{\"o}lder function over the ball $B_R(X_*)$.  Namely, in the expansion (\ref{eq:taylor_mu}) of degree $q=\lfloor \beta \rfloor$ for $f(X)$, the residual $\varepsilon_{f}(X)$ satisfies
    $\|\varepsilon_f(X)\| \le L_f \|X-X_*\|_2^{\beta}$.  
    Using constants 
    \begin{align*} 
    a := \frac{2\Gamma(1/2)^d}{\Gamma(d/2)},
    \quad
    b_{f,c} :=
    \frac{1}{2c+d}\sum_{|\bs u|=c}
    \frac{D^{2\bs u} f(X_*)}{(2\bs u)!} 
    \frac{2\prod_{j=1}^{d}\Gamma(u_j+1/2)}{\Gamma(\sum_{j=1}^{d} (u_j+1/2))},
    \quad \text{and} \quad 
    \kappa_m := \max_{j=0,1,2,\ldots,m} \binom{m}{j},
    \end{align*}
    we have 
    \begin{align}
    \frac{\integralB(f;r+\delta)-\integralB(f;r-\delta)}{\delta} 
    &=
    2 a f(X_*) r^{d-1} + 2 \sum_{c=1}^{\lfloor \beta/2 \rfloor} b_{f,c} (2c+d) r^{2c+d-1} 
    +
    \tilde{H}_{f} r^{\beta+d-1}
    +
    \tilde{H}_{f}^{\dagger} \delta^2,
    \label{eq:difference}
    \end{align}
where 
\begin{align*}
    \tilde{H}_{f} = \tilde{H}_f(r)
    &\le 
    \frac{\Gamma(1/2)^{d}}{\Gamma((d/2)+1)} 2^{\beta+1} d L_{f}, 
    \\
    \tilde{H}^{\dagger}_{f} 
    =
    \tilde{H}^{\dagger}_{f}(r) 
    &\le 
    \frac{\Gamma(1/2)^{d}}{\Gamma((d/2)+1)} 2 L_f \kappa_d R^{\beta+d-3} +
    2 \frac{a}{d} f(X_*) \kappa_d R^{d-3} + 2 \sum_{c=1}^{\lfloor \beta/2 \rfloor} b_{f,c} \kappa_{2c+d} R^{2c+d-3}.
\end{align*}
\end{prop}

\begin{proof}[Proof of Proposition~\ref{prop:integral}]
For this proof, we employ Corollary~2 in \citet{okuno2020extrapolation}:
\begin{align}
    \integralB(f;r)
    &=
    \frac{a}{d}
    f(X_*) r^d 
    +
    \sum_{c=1}^{\lfloor \beta/2 \rfloor} b_{f,c} r^{2c+d} 
    +
    \integralB(\varepsilon_{f};r).
    \label{eq:integralB}
\end{align}
Although (\ref{eq:integralB}) replaces the upper-bound evaluation of the error term ($\le L_{\beta}r^{\beta+d}\int_{B(\bs 0;1)}\diff x$; proved at the last line of Proof of Proposition~6 in \citet{okuno2020extrapolation}) by the exact term $\integralB(\varepsilon_{f};r)$, following the proof therein immediately proves (\ref{eq:integralB}). 
Using the inequality
\begin{align}
    (r+\delta)^d - (r-\delta)^d 
    &=
    \sum_{j=0}^{d} \binom{d}{j}r^{j} \delta^{d-j}
    -
    \sum_{j=0}^{d} \binom{d}{j} (-1)^{d-j} r^{j} \delta^{d-j} \nonumber \\ 
    &\le 
    2d r^{d-1} \delta 
    +
    2 \sum_{j=0}^{d-3} \binom{d}{j} r^{j} \delta^{d-j} \nonumber \\
    &\le 
    2d r^{d-1} \delta 
    +
    2 \underbrace{\max_j \binom{d}{j}}_{=\kappa_d} \left[ \sum_{j=0}^{d-3}\binom{d-3}{j} r^{j} \delta^{(d-3)-j} \right] \delta^3 \nonumber \\
    &\le 
    2d r^{d-1} \delta 
    +
    2 \kappa_d (r+\delta)^{d-3} \delta^3 \nonumber \\
    &\le
    2d r^{d-1} \delta 
    +
    2 \kappa_d R^{d-3} \delta^3 \quad (\because r+\delta \le R), \label{eq:twin_pol_difference}
\end{align}
the expression (\ref{eq:integralB}) gives
\begin{align}
    \integralB(f;r+\delta) - \integralB(f;r-\delta)
    &=
    \frac{a}{d} f(X_*) \{(r + \delta)^d-(r - \delta)^d\}
    +
    \sum_{c=1}^{\lfloor \beta/2 \rfloor} b_{f,c} \{(r + \delta)^{2c+d} -(r - \delta)^{2c+d}\} \nonumber \\
    &\hspace{15em}+
    \{\integralB(\varepsilon_{f};r + \delta) - \integralB(\varepsilon_{f};r - \delta)\} \nonumber \\
    &=
    2 \left[  
    af(X_*) r^{d-1} 
    +
    \sum_{c=1}^{\lfloor \beta/2 \rfloor} b_{f,c} (2c+d) r^{2c+d-1}
    \right] 
    \delta \nonumber \\
    &\hspace{5em}+
    \underbrace{\left[ 
    2\frac{a}{d}f(X_*) \kappa_d R^{d-3}
    +
    2 \sum_{c=1}^{\lfloor \beta/2 \rfloor} b_{f,c} \kappa_{2c+d} R^{2c+d-3} 
    \right]}_{=:H_f} \delta^3 \nonumber \\
    &\hspace{15em}+
    \{
        \integralB(\varepsilon_{f};r+\delta)
        -
        \integralB(\varepsilon_{f};r-\delta)
    \}.
    \label{eq:integralB_difference}
\end{align}
The last term in (\ref{eq:integralB_difference}) is evaluated as
\begin{align*}
    |   
    \integralB(\varepsilon_{f};r+\delta)
    -
    \integralB(\varepsilon_{f};r-\delta)
    |
&\le 
    \int_{\|X-X_*\|_2 \in [r-\delta,r+\delta]} \varepsilon_{f} (X) \diff X  \\
&\le 
    \{
        \text{vol}(B_{r+\delta}(X_*))
        -
        \text{vol}(B_{r-\delta}(X_*))
    \}
    \sup_{\|X-X_*\| \le 2r}|\varepsilon_{f} (X)| \\
&=
    \frac{\Gamma(1/2)^{d}}{\Gamma((d/2)+1)}
    \{
        (r+\delta)^d - (r-\delta)^d
    \}
    \sup_{\|X-X_*\| \le 2r}|\varepsilon_{f} (X)| \\
&\le 
    \frac{\Gamma(1/2)^{d}}{\Gamma((d/2)+1)}
    \{
        2d r^{d-1} \delta 
        +
        2\kappa_d R^{d-3} \delta^3
    \}
    L_{f} (2r)^{\beta} \\
&\le 
    \frac{\Gamma(1/2)^{d}}{\Gamma((d/2)+1)} 2^{\beta+1} d L_f r^{\beta+d-1} \delta
    +
    \frac{\Gamma(1/2)^{d}}{\Gamma((d/2)+1)} 2 L_f \kappa_d R^{\beta+d-3} \delta^3
\end{align*}
with the volume of the $d$-dimensional hypersphere $\text{vol}(B_r(X_*))=(\Gamma(1/d)^{d}/\Gamma((d/2)+1))r^d$ and $2r \le R$. 
By arranging the terms as 
\begin{align*}
    H_f \delta^3&
    +
    |   
    \integralB(\varepsilon_{f};r+\delta)
    -
    \integralB(\varepsilon_{f};r-\delta)
    | \\
    &\le
    \underbrace{\frac{\Gamma(1/2)^{d}}{\Gamma((d/2)+1)} 2^{\beta+1} d L_f }_{\ge \tilde{H}_f} r^{\beta+d-1} \delta \\
    &\hspace{5em}+
    \underbrace{\left\{ \frac{\Gamma(1/2)^{d}}{\Gamma((d/2)+1)} 2 L_f \kappa_d R^{\beta+d-3} +
    2 \frac{a}{d} f(X_*) \kappa_d R^{d-3} + 2 \sum_{c=1}^{\lfloor \beta/2 \rfloor} b_{f,c} \kappa_{2c+d} R^{2c+d-3} \right\}}_{\ge \tilde{H}_f^{\dagger}} \delta^3, 
\end{align*}
(\ref{eq:integralB_difference}) divided by $\delta$ reduces to the assertion
\begin{align*}
    \frac{\integralB(f;r+\delta)-\integralB(f;r-\delta)}{\delta} 
    &=
    2 a f(X_*) r^{d-1} + 2 \sum_{c=1}^{\lfloor \beta/2 \rfloor} b_{f,c} (2c+d) r^{2c+d-1} 
    +
    \tilde{H}_{f} r^{\beta+d-1}
    +
    \tilde{H}^{\dagger}_{f} \delta^2.
\end{align*}
\end{proof}

\begin{prop}
\label{prop:explicit_form_of_eta_r}
Let $0<r\le R/2$. 
Assume that the functions $\mu,\eta\mu$ are $\beta$-H{\"o}lder and define constants 
$a,b_{\mu,c},b_{\eta \mu}$ in the same way as $a,b_{f,c}$ in Proposition~\ref{prop:integral}. 
Then, we have 
\[
    \tilde{\eta}(r)
    =
    \eta(X_*) 
    +
    \sum_{c=1}^{\lfloor \beta/2 \rfloor} b_{\mu,\eta\mu,c}^* r^{2c}
    +
    \tilde{C}_{\mu,\eta \mu}^*(r) r^{\beta}
    \quad (\forall r \in [0,2/R]),
\]
for some constants $\{b_{\mu,\eta\mu,c}^*\}_{c=1}^{\lfloor \beta/2 \rfloor} \subset \mathbb{R}$ and 
$\sup_{r \le R/2}|\tilde{C}_{\mu,\eta\mu}^*(r)|<\infty$. 
\end{prop}

\begin{proof}[Proof of Proposition~\ref{prop:explicit_form_of_eta_r}]
Applying Proposition~\ref{prop:integral} to (\ref{eq:eta_evans}) yields
\begin{align*}
    \tilde{\eta}(r)
    &=
    \lim_{\delta \searrow 0} \tilde{\eta}_{\delta}(r) 
    =
    \lim_{\delta \searrow 0} 
    \frac{
        \{\integralB (\eta\mu;r+\delta) - \integralB(\eta\mu;r-\delta)\}/\delta
    }{
        \{\integralB (\mu;r+\delta)-\integralB(\mu;r-\delta)\}/\delta
    } \\
    &=
    \frac{    
        2 a \eta(X_*)\mu(X_*) r^{d-1} + 2 \sum_{c=1}^{\lfloor \beta/2 \rfloor} b_{\eta \mu,c} (2c+d) r^{2c+d-1} + \tilde{H}_{\eta\mu} r^{\beta+d-1}
    }{
        2 a \mu(X_*) r^{d-1} + 2 \sum_{c=1}^{\lfloor \beta/2 \rfloor} b_{\mu,c} (2c+d) r^{2c+d-1} + \tilde{H}_{\mu} r^{\beta+d-1}
    } \\
    &=
    \frac{    
        \eta(X_*)\mu(X_*) + \sum_{c=1}^{\lfloor \beta/2 \rfloor} (b_{\eta \mu,c}/a) (2c+d) r^{2c} + (\tilde{H}_{\eta\mu}/2a) r^{\beta}
    }{
        \mu(X_*) + \sum_{c=1}^{\lfloor \beta/2 \rfloor} (b_{\mu,c}/a) (2c+d) r^{2c} + (\tilde{H}_{\mu}/2a) r^{\beta}
    };
\end{align*}
expanding the last fraction proves the assertion. 
The finiteness of $\sup_{r \le R/2}|\tilde{C}^*_{\mu,\eta\mu}(r)|$ is obtained by the (uniform) boundedness of all the coefficients, $\tilde{H}_{\mu}=\tilde{H}_{\mu}(r)$ and $\tilde{H}_{\mu\eta}=\tilde{H}_{\mu\eta}(r)$ over $r \in [0,R/2]$. 
\end{proof}

\begin{prop}
\label{prop:lower_pn}
There exists $c>0$ such that $\integralB(\mu;\tilde{r}_n) \ge c \tilde{r}_n^d$. 
\end{prop}

\begin{proof}
$\integralB(\mu;\tilde{r}_n) > l \text{vol}(B_{\tilde{r}_n}(X_*)) \ge lc' \tilde{r}_n^d=:c \tilde{r}_n^d$ as $\mu(X)$ is lower-bounded by $l>0$ and $\text{vol}(B_{r}(X_*))=c' r^d$ for some $c'>0$. 
\end{proof}

\begin{prop}
\label{prop:tail_Nn}
There exists $c>0$ such that $\mathbb{P}(|N_n-\overline{N}_n| \ge \overline{N}_n/2) \le 2\exp\left(-c n\tilde{r}_n^d \right)$. 
\end{prop}

\begin{proof}
As $N_n$ follows a Binomial distribution with parameters $n$ and $\tilde{p}_n:=\integralB(\mu;\tilde{r}_n)$, (a generalization of) Chernoff bound proves the bound 
\[
    \mathbb{P}(|N_n-\overline{N}_n| \ge \overline{N}_n/2) \le 2\exp\left(-\frac{c' \overline{N}_n}{4} \right)
\]
for some $c'>0$. See, e.g., \citet{vershynin2018high} Exercise~2.3.5. 
As Proposition~\ref{prop:lower_pn} proves 
$\overline{N}_n = n \tilde{p}_n \ge c'' n \tilde{r}_n^d$ for some $c''>0$, 
taking $c:=c'c''/4$ proves the assertion. 
\end{proof}

\begin{prop}
\label{prop:negation_probability}
There exist $c,c',c''>0$ such that $\mathbb{P}(\lnot \event) \le 2\exp(-c n\tilde{r}_n^d)+c' \exp(-c'' n^{\tau})$.
\end{prop}
\begin{proof}
Let $\omega=\lfloor \beta/2 \rfloor$.
Applying the inequalities in Proposition~\ref{prop:tail_Nn} and Condition (C-3) to the rightmost side of 
\[
\mathbb{P}(\lnot \event) 
\, = \, 
\mathbb{P}(N_n \le \omega \text{ or }\zeta_n >\varphi N_n) 
\, \le \, 
\mathbb{P}(N_n \le \omega)+ \mathbb{P}(\zeta_n>\varphi N_n) 
\, \le \,
\mathbb{P}(|N_n-\overline{N}_n| \ge \overline{N}_n/2)+\mathbb{P}(\zeta_n > \varphi N_n)
\]
proves the assertion. 
\end{proof}

\begin{prop}
\label{prop:prob_Nnomega}
There exist $c,c'>0$ and $n' \in \mathbb{N}$ such that 
$\mathbb{E}\left( 1/N_n \mid \event \right) \le 4\exp(-cn\tilde{r}_n^d) + \frac{4}{c' n \tilde{r}_n^d}$ for $n \ge n'$.
\end{prop}

\begin{proof}[Proof of Proposition~\ref{prop:prob_Nnomega}]
Let $n'$ be a sufficiently large constant and let $n>n'$.
\begin{align*}
    \mathbb{E}\left( \frac{1}{N_n} \mid \event \right)
    &= 
    \sum_{N=\omega+1}^{\infty} \frac{1}{N} \mathbb{P}(N_n=N \mid \event) 
    \le 
    \sum_{N=\omega+1}^{\infty} \frac{1}{N} \frac{\mathbb{P}(N_n=N)}{\mathbb{P}(\event)} 
    \le 
    2\sum_{N=\omega+1}^{\infty} \frac{1}{N} \mathbb{P}(N_n=N),  
\end{align*}
where the last inequality follows from 
$1 \ge \mathbb{P}(\event)=1-\mathbb{P}(\lnot \event) \ge 1/2$ proved by Proposition~\ref{prop:tail_Nn}. 
As Proposition~\ref{prop:lower_pn} proves $\overline{N}_n/2 = n \integralB(\mu;\tilde{r}_n)/2 \ge cn \tilde{r}_n^d \ge \omega$, the last term is also evaluated as
\begin{align*}
    \sum_{N=\omega+1}^{\infty} \frac{1}{N} \mathbb{P}(N_n=N)
    &=
    \sum_{\substack{N \ge \omega+1 \\ |N-\overline{N}_n| \ge \overline{N}_n/2}} \frac{1}{N} \mathbb{P}(N_n=N)
    +
    \sum_{\substack{N \ge \omega+1 \\ |N-\overline{N}_n| < \overline{N}_n/2}} \frac{1}{N} \mathbb{P}(N_n=N) \\
    &\le 
    \sum_{\substack{N \ge \omega+1 \\ |N-\overline{N}_n| \ge \overline{N}_n/2}} \mathbb{P}(N_n=N)
    +
    \frac{1}{\overline{N}_n/2}
    \sum_{\substack{N \ge \omega+1 \\ |N-\overline{N}_n| < \overline{N}_n/2}} \mathbb{P}(N_n=N) \\
    &\le 
    \mathbb{P}(|N_n-\overline{N}_n| \ge \overline{N}_n/2)
    +
    \frac{2}{\overline{N}_n} \\
        &\le 
        2\exp(-cn\tilde{r}_n^d) + \frac{2}{c' n \tilde{r}_n^d}
    \quad (\because \: \text{Proposition}~\ref{prop:lower_pn} \text{ and } \ref{prop:tail_Nn} );
    \end{align*}
    the assertion is proved. 
\end{proof}

\subsection{Proof of Example~\ref{ex:C-3}}
\label{proof:ex:C-3}

As $\omega=\lfloor \beta/2 \rfloor=1$, we regard $\tilde{R}_n \in \mathbb{R}^{N_n \times 1}$ as a vector of length $N_n$ for notation simplicity. 
$\zeta_n$ then reduces to 
\[
    \zeta_n 
    = 
    \langle 1_{N_n},\tilde{R}_n(\tilde{R}_n^{\top}\tilde{R}_n)^{-1}\tilde{R}_n^{\top}1_{N_n}\rangle 
    =
    \frac{\langle 1_{N_n}, \tilde{R}_n \rangle^2}{\|\tilde{R}_n\|_2^2}
    =
    N_n \frac{(N_n^{-1}\sum_{i=1}^{N_n} r_i)^2}{N_n^{-1}\sum_{i=1}^{N_n}r_i^2}
    =:
    N_n \frac{(\xi_n^{(1)})^2}{\xi_n^{(2)}},
\]
and $\mathbb{P}(\zeta_n \ge \varphi N_n) = \mathbb{P}((\xi_n^{(1)})^2/\xi_n^{(2)} \ge \varphi)$ holds. Therefore, it suffices to show the exponential concentration of $\xi_n^{(1)}$ and $\xi_n^{(2)}$, and 
\[
\varphi > \varrho = \frac{(\lim_{n \to \infty}\xi_n^{(1)}/\tilde{r}_n)^2}{\lim_{n \to \infty}\xi_n^{(2)}/\tilde{r}_n^2}
=
1-\frac{1}{(d+1)^2}
\quad 
(\text{with the scaling term }\tilde{r}_n).
\]

\paragraph{Moments of $r$.} 
Under the condition $N_n=N$ with fixed $\tilde{r}_n$, the radii $\{r_1,r_2,\ldots,r_N\}$ can be regarded as $N$ copies of a random variable $r$ following a conditional distribution $\mathbb{P}(r=\|X-X_*\|_2 \mid r \le \tilde{r}_n)$.

Considering that $\mu$ follows uniform distribution, the conditional probability density $\tilde{\eta}(r)$ is proportional to the area of the ball surface; we have $\tilde{\eta}(r) \propto r^{d-1}$ and thus we obtain 
\[
    \mathbb{E}(r^{k})
    =
    \left\{ 
        \int_0^{\tilde{r}_n} r^k r^{d-1} \diff r 
    \right\}
    \big/ 
    \left\{
        \int_0^{\tilde{r}_n} r^{d-1} \diff r
    \right\}
    =
    \frac{d}{d+k} \tilde{r}_n^k
    \quad (k=1,2). 
\]
Namely, $\mathbb{E}(r^{k})$ is of order $\tilde{r}_n^k$, and this ensures the following concentration of $\xi_n^{(1)},\xi_n^{(2)}$.

\paragraph{Concentration of $\xi_n^{(1)}$ and $\xi_n^{(2)}$.}

As $r_1,r_2,\ldots,r_{N_n} \in [0,\tilde{r}_n]$, Hoeffding inequality proves 
\[
    \mathbb{P}\left( 
        | \xi_n^{(k)} - \mathbb{E}(r_1^k) | 
        \ge \alpha \tilde{r}_n^k
        \mid N_n=N
    \right)
    \le 
    2 \exp(-2 \alpha^2 N)
\]
for $\alpha > 0$ and $k=1,2$. With sufficiently small $\alpha$ and $\xi_n^{(1)},\xi_n^{(2)} \in (0,1)$, we have 
\begin{align*}
    &\mathbb{P}\left(
        \frac{(\xi_n^{(1)})^2}{\xi_n^{(2)}}
        \in 
        \left[
            \frac{(\mathbb{E}(r_1)-\alpha \tilde{r}_n)^2}{\mathbb{E}(r_1^2) + \alpha \tilde{r}_n^2},
            \frac{(\mathbb{E}(r_1)+\alpha \tilde{r}_n)^2}{\mathbb{E}(r_1^2) - \alpha \tilde{r}_n^2}
        \right]
        \mid N_n=N
    \right) \\
    &\hspace{5em}
    \ge 
    \mathbb{P}(| \xi_n^{(1)} - \mathbb{E}(r_1) | 
        \le \alpha \tilde{r}_n
    \mid N_n=N) \,
    \mathbb{P}(| \xi_n^{(2)} - \mathbb{E}(r_1^2) | 
        \le \alpha \tilde{r}_n^2
    \mid N_n=N)
\end{align*}
indicating that, 
\begin{align}
    \mathbb{P}\left(
        \frac{(\xi_n^{(1)})^2}{\xi_n^{(2)}}
        \notin 
        \left[
            \frac{(\mathbb{E}(r_1)-\alpha \tilde{r}_n)^2}{\mathbb{E}(r_1^2) + \alpha \tilde{r}_n^2},
            \frac{(\mathbb{E}(r_1)+\alpha \tilde{r}_n)^2}{\mathbb{E}(r_1^2) - \alpha \tilde{r}_n^2}
        \right]
        \mid N_n=N
    \right) 
    \le 
    1-(1-2\exp(-2\alpha^2 N))^2
    \le 
    C \exp(-2\alpha^2 N)
    \label{eq:tail_fraction}
\end{align}
for some $C>0$. 
Therefore, by specifying $\tilde{\alpha}>0$ such that 
\[
        \left[
            \frac{(\mathbb{E}(r_1)-\alpha \tilde{r}_n)^2}{\mathbb{E}(r_1^2) + \alpha \tilde{r}_n^2},
            \frac{(\mathbb{E}(r_1)+\alpha \tilde{r}_n)^2}{\mathbb{E}(r_1^2) - \alpha \tilde{r}_n^2}
        \right]
        \subset 
        [\varrho - \tilde{\alpha},\varrho + \tilde{\alpha}], 
        \quad 
        \varrho := \frac{\mathbb{E}(r_1)^2}{\mathbb{E}(r_1^2)} = 1-\frac{1}{(d+1)^2},
\]
we have 
\[
    \mathbb{P}\left(
        \bigg| \frac{(\xi_n^{(1)})^2}{\xi_n^{(2)}} - \varrho \bigg| \ge \tilde{\alpha}
        \mid N_n=N
    \right)
    \le 
    \mathbb{P}\left(
        \frac{(\xi_n^{(1)})^2}{\xi_n^{(2)}}
        \notin 
        \left[
            \frac{(\mathbb{E}(r_1)-\alpha \tilde{r}_n)^2}{\mathbb{E}(r_1^2) + \alpha \tilde{r}_n^2},
            \frac{(\mathbb{E}(r_1)+\alpha \tilde{r}_n)^2}{\mathbb{E}(r_1^2) - \alpha \tilde{r}_n^2}
        \right]
        \mid N_n=N
    \right) 
    \le 
    C\exp(-2\alpha^2 N), 
\]
Combined with the concentration inequality of $N_n$ shown in Proposition~\ref{prop:tail_Nn}, we have 
\begin{align*}
    \mathbb{P}\left(
        \bigg| \frac{(\xi_n^{(1)})^2}{\xi_n^{(2)}} - \varrho \bigg| \ge \tilde{\alpha}
    \right)
    &=
    \sum_{N=0}^{\infty}
    \mathbb{P}\left(
        \bigg| \frac{(\xi_n^{(1)})^2}{\xi_n^{(2)}} - \varrho \bigg| \ge \tilde{\alpha} \mid N_n=N
    \right)
    \mathbb{P}(N_n=N) \\
    &=\sum_{N:|N-\overline{N}_n| < \overline{N}_n/2}
    \mathbb{P}\left(
        \bigg| \frac{(\xi_n^{(1)})^2}{\xi_n^{(2)}} - \varrho \bigg| \ge \tilde{\alpha} \mid N_n=N
    \right)
    \mathbb{P}(N_n=N) \\
    &\hspace{5em}+
    \sum_{N:|N-\overline{N}_n| \ge \overline{N}_n/2}
    \mathbb{P}\left(
        \bigg| \frac{(\xi_n^{(1)})^2}{\xi_n^{(2)}} - \varrho \bigg| \ge \tilde{\alpha} \mid N_n=N
    \right)
    \mathbb{P}(N_n=N) \\
    &\le 
    \sup_{N:|N-\overline{N}_n| < \overline{N}_n/2}
    \mathbb{P}\left(
        \bigg| \frac{(\xi_n^{(1)})^2}{\xi_n^{(2)}} - \varrho \bigg| \ge \tilde{\alpha} \mid N_n=N
    \right)
    +
    \mathbb{P}(|N_n-\overline{N}_n| \ge \overline{N}_n/2) \\
    &\le 
    2\exp(-\alpha \overline{N}_n) + 2\exp(-cn\tilde{r}_n^d) \\
    &\le 
    4\exp(-c' n\tilde{r}_n^d) \\
    &\le 
    4 \exp(-c' n^{\nu})
\end{align*}
for some $c',\nu>0$, as $\overline{N}_n \propto n \tilde{r}_n^d$ and (\ref{eq:rate_rn}). 
By specifying sufficiently small $\alpha>0$, we can take arbitrarily small $\tilde{\alpha}=\tilde{\alpha}(\alpha)>0$; with $\varphi=1-\frac{1}{2(d+1)^2}(>\varrho)$ and $\tilde{\alpha} \in (0,1/\{2(d+1)^2\})$, the assertion is proved by
\[
    \mathbb{P}(\zeta_n \ge \varphi N_n)
    =
    \mathbb{P}\left(
        \frac{(\xi_n^{(1)})^2}{\xi_n^{(2)}}
        \ge \varphi
    \right)
    \le 
    \mathbb{P}\left(
        \bigg| \frac{(\xi_n^{(1)})^2}{\xi_n^{(2)}} - \varrho \bigg| \ge \tilde{\alpha}
    \right)
    \le 
    4\exp(-c'n^{\nu}).
\]


\end{document}